\documentclass[twoside]{article}
\usepackage{medmath}
\usepackage[accepted]{aistats2025}
\usepackage{lipsum, babel}
\usepackage{tcolorbox}

\usepackage[utf8]{inputenc} 
\usepackage[T1]{fontenc}    
\usepackage{hyperref}       
\usepackage{natbib}
\usepackage{booktabs}       
\usepackage{amsfonts}       
\usepackage{nicefrac}       
\usepackage{microtype}      
\usepackage{xcolor}         
\usepackage{algpseudocode}
\usepackage{algorithm}
\usepackage{lipsum}
\usepackage{wrapfig}
\usepackage{mathtools}
\usepackage{amsthm}
\usepackage{amsmath,amsfonts,amssymb,color}
\usepackage{mathtools}
\usepackage{dsfont}
\usepackage{epsfig}
\usepackage{epstopdf}
\usepackage{pifont}
\usepackage{caption}
\usepackage{subcaption}
\usepackage{tcolorbox}
\tcbuselibrary{skins,breakable}
\usepackage{tikz}
\usetikzlibrary{automata,arrows,positioning,calc}
\usetikzlibrary{decorations.pathreplacing}

\title{Adversarially-Robust TD Learning with Markovian Data: Finite-Time Rates and Fundamental Limits}
\pdfoutput=1 
\usepackage[utf8]{inputenc} 
\usepackage[T1]{fontenc}    
\usepackage{hyperref}       
\usepackage{url}            
\usepackage{booktabs}       
\usepackage{amsfonts}       
\usepackage{nicefrac}       
\usepackage{microtype}      
\usepackage{xcolor}         
\usepackage{algpseudocode}
\usepackage{algorithm}
\usepackage{lipsum}
\usepackage{wrapfig}
\usepackage{mathtools}
\usepackage{amsthm}
\usepackage{amsmath,amsfonts,amssymb,color}
\usepackage{mathtools}
\usepackage{dsfont}
\usepackage{epsfig}
\usepackage{epstopdf}
\usepackage{pifont}
\newcommand{\mc}{\mathcal}

\newtheorem{definition}{Definition}
\newtheorem{theorem}{Theorem}
\newtheorem{lemma}{Lemma}

\newtheorem{assumption}{Assumption}

\DeclarePairedDelimiter\ceil{\lceil}{\rceil}
\DeclarePairedDelimiter\floor{\lfloor}{\rfloor}

\DeclarePairedDelimiterX{\norm}[1]{\lVert}{\rVert}{#1}

\definecolor{mygreen}{rgb}{0.0, 0.5, 0.0}
\definecolor{winered}{rgb}{0.8,0,0}
\definecolor{myblue}{rgb}{0,0,0.8}

\hypersetup{
    colorlinks=true,
    linkcolor={winered},
    citecolor={myblue}
}
\author{%
  David S.~Hippocampus\thanks{Use footnote for providing further information
    about author (webpage, alternative address)---\emph{not} for acknowledging
    funding agencies.} \\
  Department of Computer Science\\
  Cranberry-Lemon University\\
  Pittsburgh, PA 15213 \\
  \texttt{hippo@cs.cranberry-lemon.edu} \\
}
\begin{document}

%

%

\twocolumn[

\aistatstitle{Adversarially-Robust TD Learning with Markovian Data: Finite-Time Rates and Fundamental Limits}

\aistatsauthor{ Sreejeet Maity \And Aritra Mitra}

\aistatsaddress{ North Carolina State University \And North Carolina State University} ]

\begin{abstract}
One of the most basic problems in reinforcement learning (RL) is policy evaluation: estimating the long-term return, i.e., value function, corresponding to a given fixed policy. The celebrated Temporal Difference (TD) learning algorithm addresses this problem, and recent work has investigated finite-time convergence guarantees for this algorithm and variants thereof. However, these guarantees hinge on the reward observations being always generated from a well-behaved (e.g., sub-Gaussian) true reward distribution. Motivated by harsh, real-world environments where such an idealistic assumption may no longer hold, we revisit the policy evaluation problem from the perspective of \emph{adversarial robustness}. In particular, we consider a Huber-contaminated reward model where an adversary can arbitrarily corrupt each reward sample with a small probability $\epsilon$. Under this observation model, we first show that the adversary can cause the vanilla TD algorithm to converge to any arbitrary value function. We then develop a novel algorithm called \texttt{Robust-TD} and prove that its finite-time guarantees match that of vanilla \texttt{TD} with linear function approximation up to a small $O(\epsilon)$ term that captures the effect of corruption. We complement this result with a minimax lower bound, revealing that such an additive corruption-induced term is unavoidable. To our knowledge, these results are the first of their kind in the context of adversarial robustness of stochastic approximation schemes driven by Markov noise. The key new technical tool that enables our results is an analysis of the Median-of-Means estimator with corrupted, time-correlated data that might be of independent interest to the literature on robust statistics.  
\end{abstract}

\section{Introduction} 
In recent years, a significant body of research has focused on understanding the effects of adversarial corruption on deep learning~\citep{goodfellow, madry}. While this line of work has contributed significantly to the design of reliable and trustworthy machine-learning models, the developments have primarily catered to supervised learning~\citep{javanmard}. Much less is understood when an adversary poisons data arriving in an online manner in the context of reinforcement learning (RL). Arguably, one of the most fundamental problems in RL is that of \emph{policy evaluation}, where a learner unaware of the true underlying model of a Markov Decision Process (MDP) seeks to evaluate the long-term return (i.e., the value function) associated with a given fixed policy. To do so, at each time step, it plays an action based on the policy to be evaluated, observes as data a reward, and transitions to a new state. Importantly, the rewards are always generated based on the (unknown) reward functions of the MDP. Departing from this paradigm, we consider a scenario where a small fraction of the reward samples can be \emph{arbitrarily} corrupted by a powerful adversary possessing complete knowledge of the MDP. One would ideally like to obtain guarantees on value function estimation that \emph{degrade gracefully} with the corruption fraction. Whether this is possible is a hitherto unexplored question that we resolve in this paper.

To provide context, in the absence of adversarial corruption, the classical Temporal Difference (TD) learning algorithm of~\cite{sutton1988} solves the policy evaluation problem. An asymptotic analysis of \texttt{TD}(0) - the simplest TD learning algorithm - with linear function approximation was provided in the seminal work of~\cite{tsitsiklisroy}. More recently, a growing body of work has provided finite-time guarantees for TD learning with linear function approximation~\citep{dalal, bhandari_finite, srikant, patil2023,  mitra2024simple}, and more general nonlinear stochastic approximation (SA) schemes~\citep{chenQ, chen2022Auto, Qu, chen2024lyapunov}. The guarantees in each of the papers above assume that the rewards are always drawn from true reward distributions linked to the underlying MDP. Moreover, the rewards are either assumed to be deterministic or generated from light-tailed sub-Gaussian distributions. 

\textbf{Motivation.} Unfortunately, such assumptions do not adequately capture harsh, real-world environments. For instance, in large-scale, complex systems such as the power grid~\citep{kosut} or multi-robot networks~\citep{gil}, data is collected via imperfect sensors prone to unexpected failures and adversarial attacks. Motivated by the need to safeguard against such attacks that are common in cyber-physical  systems~\citep{dibaji}, we consider a reward contamination model where, at each time step, with probability $1-\epsilon$, the reward is generated from a true reward distribution, and with probability $\epsilon$, from an arbitrary error distribution controlled by an adversary. Here, $\epsilon$ captures the power of the adversary. Our data poisoning model is directly inspired by the Huber model from robust statistics~\citep{huber, huber2}. Similar reward contamination models have also been widely studied in the context of multi-armed bandits~\citep{junadv, lykouris, liuadv, guptaadv, kapoor}. However, beyond bandits, when it comes to \emph{SA schemes in RL}, no prior work has provided a finite-time analysis of the effects of such attacks. Given this premise, we ask:

\textit{Is it possible to perform accurate value function estimation under the Huber-contaminated reward model? If so, what are the fundamental limits on performance imposed by this attack model?} 

The main difficulty in answering these questions arises from the need to deal with two different forms of uncertainty: the lack of knowledge of the MDP, and the uncertainty injected by the adversary. Furthermore, other than requiring the true reward distributions to have finite first and second moments, we make no assumptions of sub-Gaussanity. This makes it particularly challenging for the learner to distinguish between time-correlated, potentially heavy-tailed clean rewards (inliers) and adversarial outliers. 

\textbf{Our Contributions.} In this paper, we systematically study adversarial robustness in the context of policy evaluation with linear function approximation. Our specific contributions are as follows. 

$\bullet$ \textbf{Vulnerability of \texttt{TD}}. We start with a simple result (Theorem~\ref{thm:vulnerability}) showing that under the Huber-contaminated reward model, an adversary can cause the vanilla \texttt{TD}(0) algorithm to converge to any arbitrary point. This result directly motivates the need for adversarially robust variants of \texttt{TD}(0). 

$\bullet$ \textbf{Robust Mean Estimation with Markov Data}. A key ingredient in our algorithmic development is that of robust mean estimation. While the literature on robust statistics has made significant advances in this regard~\citep{lai, chenrobust}, the results we know of all assume independent and identically (i.i.d.) distributed inliers. The same is true for some recent papers in RL~\citep{zhuanghv, zhu2024hv} that consider heavy-tailed i.i.d. rewards with no corruption. Since the reward samples in our setting are generated based on a Markov chain, we cannot directly appeal to such existing work. To overcome this difficulty, we provide the \emph{first analysis of the Median-of-Means (\texttt{MoM}) estimator under Huber contamination and Markovian data.} In particular, our analysis carefully exploits the ergodicity of the underlying Markov chain along with a novel coupling argument. We also note that while the popular \texttt{MoM} scheme was known to be robust to heavy-tailed data, the fact that it is also robust to adversarial corruption appears to be new. As such, we believe that our main result in this context, namely Theorem~\ref{theorem:RobustMean}, and its analysis in Appendix~\ref{app:proofrobustmeanest}, might be of independent interest to robust statistics. 

$\bullet$ \textbf{Robust-TD Algorithm}. On the algorithmic front, our main contribution is the development of an adversarially robust variant of \texttt{TD}(0) called \texttt{Robust-TD}. \texttt{Robust-TD} relies on two main new ideas: (i) a robust mean estimation step that uses historical data to construct robust empirical TD update directions; and (ii) a dynamic thresholding step that provides a second layer of safety by accounting for low-probability events where the robust mean estimation guarantees might fail to hold. As we discuss in detail in Section~\ref{sec:robustTD}, the design of the thresholding radius is a very delicate task: unless designed carefully, one may not achieve the near-optimal rates in this paper.

$\bullet$ \textbf{Main Convergence Result.} Our main convergence result for \texttt{Robust-TD} establishes a mean-square error bound of the form $\tilde{O}(\bar{\tau}_{mix}/T) + O(\epsilon)$, where $\bar{\tau}_{mix}$ is the mixing time of the underlying Markov chain, $T$ is the number of iterations, and $\epsilon$ is the corruption probability. For a specific statement of this result, see Theorem~\ref{thm:main_result}. When $\epsilon=0$, our result is consistent with prior finite-time guarantees for \texttt{TD}(0) with linear function approximation~\citep{bhandari_finite, srikant}. Thus, \texttt{Robust-TD} is provably robust to adversarial reward contamination, and its guarantees match that of vanilla \texttt{TD}(0) up to a small $O(\epsilon)$ term. Establishing this result is non-trivial as we need to contend with the complex interplay between Markovian noise, adversarial perturbations, and function approximation. We elaborate on these challenges in Sections~\ref{sec:model} and~\ref{sec:results}. To our knowledge, Theorem~\ref{thm:main_result} is the first result of its kind for stochastic approximation schemes in RL driven by Markov noise and subject to adversarial outliers. 

$\bullet$ \textbf{Minimax Lower Bound.} To complement the upper-bound in Theorem~\ref{thm:main_result}, we provide an algorithm-independent lower bound in Theorem~\ref{thm:lowerbnd}. The main message conveyed by this result is that the additive $O(\epsilon)$ term is \emph{unavoidable}, and captures the fundamental price of adversarial contamination. 

Overall, our algorithmic and theoretical contributions above provide a fairly complete characterization of the effects of Huber-reward-contamination on policy evaluation in general, and TD learning in particular. 


\textbf{Related Work.} We discuss the most relevant works on \emph{adversarial robustness in RL} below. A more elaborate survey appears in Appendix~\ref{app:add_lit}. Data corruption in online finite-horizon episodic RL problems is studied by~\cite{lykourisRL} and~\cite{wei2022model}, where the notion of performance is measured by cumulative regret. Our setting is \emph{fundamentally different} in that we consider an infinite horizon, discounted single-trajectory setting, where performance is captured by the mean-squared error w.r.t. the solution to the projected Bellman equation. Furthermore, our algorithm builds on stochastic approximation and differs significantly from the Upper-Confidence-Based (UCB)/Action-Elimination type algorithms employed in~\cite{lykourisRL, wei2022model}. Corruption-robust algorithms in the offline setting are considered in~\cite{zhang2022corruption}, where data tuples are collected offline in an i.i.d. manner, and the true rewards are assumed to be sub-Gaussian. In sharp contrast, data arrives sequentially in our setting, and, as such, \emph{we need to contend with corruption in heavy-tailed Markovian data} - a much more challenging setting. Different from the SA problem we consider here, outlier-robust policy gradient (PG) algorithms have been explored in~\cite{zhang2021robust}, where the issue of Markovian sampling does not arise. Finally, a very recent work~\citep{cayci2024} considers heavy-tailed rewards \emph{with no adversarial corruption} in TD learning. The analysis in their paper requires a strong realizability assumption and relies on a projection step in the algorithm to control the iterates; we require neither, making it much more challenging to tackle both heavy-tailed data and adversarial perturbations. Moreover, our proposed algorithm differs considerably from that in~\cite{cayci2024}. \textbf{In summary, our work is the first to study the topic of adversarial reward corruption in the context of TD learning with function approximation and Markovian data.}  As such, we do not focus on other potential forms of attack. In fact, as we shall see, the reward attack model we consider here is rich enough to merit non-trivial and subtle algorithmic ideas and analysis techniques. 

\section{Model and Problem Formulation}
\label{sec:model}
\vspace{-2.5 mm}
We start by reviewing the essentials of policy evaluation with linear function approximation, and then proceed to set up our problem of interest. 

\textbf{The Policy Evaluation Problem.} We consider a Markov Decision Process (MDP) denoted by $\mc{M}=(\mc{S},\mc{A},\mc{P},\mc{R},\gamma)$,  where $\mc{S}$ is a finite state space of size $m$, $\mc{A}$ is a finite action space, $\mc{P}$ is a set of action-dependent Markov transition kernels, $\mc{R}$ is a reward function, and $\gamma \in (0,1)$ is the discount factor. We consider deterministic policies, where each deterministic policy $\mu: \mc{S} \rightarrow \mc{A}$ maps a state to an action.\footnote{The assumption of deterministic policies is only for simplicity of exposition. Our results can be easily extended to account for stochastic policies.} A fixed policy $\mu$ induces a Markov reward process (MRP) characterized by a transition matrix $P_{\mu}$, and a reward function $R_{\mu}:\mc{S} \rightarrow \mathbb{R}$. Here, $P_{\mu}(s,s')$ denotes the probability of transitioning from state $s$ to state $s'$ under the action $\mu(s)$. Associated with each state $s \in \mc{S}$ is a conditional reward distribution $\mc{D}_{\mu}(\cdot|s)$: whenever action $\mu(s)$ is played in state $s$, a noisy \emph{random} reward $r(s)$ drawn from $\mc{D}_{\mu}(\cdot|s)$ is observed, such that $\mathbb{E}_{r(s) \sim \mc{D}_{\mu}(\cdot|s)}[r(s)]= R_\mu(s)$, and $\mathbb{E}_{r(s) \sim  \mc{D}_{\mu}(\cdot|s)}[(r(s)-R_{\mu}(s))^2] \leq \rho^2$, where $\rho$ is assumed to be finite. In words, upon playing $\mu(s)$ in state $s$, the observed noisy reward has mean $R_\mu(s)$ and variance upper-bounded by $\rho^2.$ We assume that the mean reward at each state is uniformly bounded, i.e., $\exists \bar{r} > 0$ such that $|R_{\mu}(s)| \leq \bar{r}, \forall s \in \mc{S}.$
The long-term value of a state $s$ in the MRP induced by $\mu$ is captured by a value function $V_{\mu}(s)$. Formally, $V_{\mu}(s)$ is the discounted expected cumulative reward obtained by playing policy $\mu$ starting from initial state $s$:
\begin{equation}
    V_{\mu}(s) = \mathbb{E}\left[ \sum_{t=0}^{\infty}\gamma^t R_{\mu}(s_t) | s_0 = s \right],
\label{eqn:v_cum}
\end{equation}
where $s_t$ represents the state of the Markov chain (induced by $\mu$) at time $t$, when initiated from $s_0=s$. The primary goal of this paper is to study the \emph{policy evaluation} problem, i.e., the problem of evaluating the value function $V_{\mu}$ corresponding to a given policy $\mu$, when $R_{\mu}$ and $P_{\mu}$ are \emph{unknown} to the learner. 

\textbf{Linear Function Approximation.} In practice, the size of the state space $\mc{S}$ can be  extremely large. This renders the task of estimating $V_\mu$ \textit{exactly} (based on observations of rewards and state transitions) intractable. One common approach to tackle this difficulty is to consider a parametric approximation $\hat{V}_\theta$ of $V_\mu$ in the linear subspace spanned by a set $\{\phi_k\}_{k\in [K]}$ of $K \ll m$  basis vectors, where $\phi_k =[\phi_k(1), \ldots, \phi_k(m)]^{\top} \in \mathbb{R}^{m}$. Specifically, we have 
$ \hat{V}_\theta(s) = \sum_{k=1}^{K} \theta(k)\phi_k(s),$ 
where $\theta = [\theta(1), \ldots, \theta(K)]^{\top} \in \mathbb{R}^{K}$ is a weight vector. Let $\Phi \in \mathbb{R}^{m \times K}$ be a matrix with $\phi_k$ as its $k$-th column; we will denote the $s$-th row of $\Phi$ by $\phi(s) \in \mathbb{R}^{K}$, and refer to it as the feature vector for state $s$. Compactly, $\hat{V}_\theta=\Phi \theta$, and for each $s\in \mc{S}$, we have that $\hat{V}_\theta(s) = \langle \phi(s), \theta \rangle$. As is standard~\citep{bhandari_finite, srikant}, we assume that the columns of $\Phi$ are linearly independent, and that the feature vectors are normalized, i.e., for each $s \in \mc{S}$, $\Vert \phi(s) \Vert^2_2 \leq 1.$ 

\textbf{The \texttt{TD}(0) Algorithm.} Given the above setup, the goal is to find the best parameter vector $\theta^*$ that minimizes the distance (in a suitable norm) between $\hat{V}_{\theta}$ and $V_\mu$. To achieve this, we will focus on the classical \texttt{TD}(0) algorithm~\cite{sutton1988} within the family of TD learning algorithms. At each time-step $t=0, 1, \ldots$, this algorithm receives as observation a data tuple $X_t=(s_t, s_{t+1}, r_t=r(s_t))$  comprising of the current state $s_t$, the next state $s_{t+1}$ reached by playing action $\mu(s_t)$, and the instantaneous reward $r_t \sim  \mc{D}_{\mu}(\cdot|s_t)$. Next, we define the \texttt{TD}(0) update direction $g_t(\theta)=g(X_t,\theta)$ as:
\begin{equation}
g_t(\theta) \triangleq \left(r_t + \gamma \langle \phi(s_{t+1}), \theta \rangle -  \langle \phi(s_{t}), \theta \rangle\right) \phi(s_t), \forall \theta \in \mathbb{R}^K. 
\nonumber
\end{equation}
The \texttt{TD}(0) update to the current parameter $\theta_t$ then takes the following form: 
\begin{equation}
    \theta_{t+1}=\theta_t + \alpha_t g_t(\theta_t),
\label{eqn:TD(0)update}
\end{equation}
where $\alpha_t \in (0,1)$ is the step-size/learning rate. Under some mild technical conditions, it was shown in \cite{tsitsiklisroy} that the iterates generated by \texttt{TD}(0) converge almost surely to the unique solution $\theta^*$ of the projected Bellman equation $\Pi_D \mc{T}_\mu (\Phi \theta^*) = \Phi \theta^*$. Here, $D$ is a diagonal matrix with entries given by the elements of the stationary distribution $\pi$ of the kernel $P_{\mu}$. Moreover, $\Pi_D(\cdot)$ is the projection operator onto the subspace spanned by $\{\phi_k\}_{k\in [K]}$ with respect to the inner product $\langle \cdot, \cdot \rangle_D$.

\textbf{Policy Evaluation with a Corrupted Reward Model.} We depart from the standard policy evaluation setting reviewed above by considering an observation model where an adversary can \emph{arbitrarily} perturb a small fraction $\epsilon \in [0, 1/2)$ of the rewards as per the classical Huber contamination model in robust statistics~\citep{huber, huber2, lai, chenrobust}. Specifically, at each time-step $t$, a reward $\tilde{r}_t$ is generated as follows. First, a Bernoulli random variable $Z_t$ that takes value $1$ with probability $(1-\epsilon)$ and value $0$ with probability $\epsilon$ is generated independently of all prior history. If $Z_t = 1$, then $\tilde{r}_t$ is sampled from the true reward distribution $\mc{D}_\mu(\cdot|s_t)$. If $Z_t=0$, then $\tilde{r}_t$ is generated from an unconstrained and unknown error distribution $\mc{Q}$ controlled by an adversary. Compactly, we have $\tilde{r}_t \sim (1-\epsilon) \mc{D}_\mu(\cdot|s_t) + \epsilon \mc{Q}$, where we use $(1-\epsilon) \mc{P}_1 + \epsilon \mc{P}_2$ to denote the mixture of two distributions $\mc{P}_1$ and $\mc{P}_2$. Here, $\epsilon$ is the proportion of contamination and captures the power of the adversary. The distribution $\mc{Q}$ could potentially be both state- and time-dependent; furthermore, the adversarial bias injected when $\tilde{r}_t$ is sampled from $\mc{Q}$ is allowed to be \emph{arbitrary}. Let us note that under the corrupted observation model, the learner is presented with a modified sequence $\{\tilde{X}_t\}$ of observations, where $\tilde{X}_t = (s_t, s_{t+1}, \tilde{r}_t).$ Given this setup, we are interested in providing precise answers to the following questions.

\begin{itemize}
\item[Q1.] \emph{Under the corrupted observation model, can we still hope to obtain a reliable estimate of the value function $V_{\mu}?$ If so, how can this be achieved algorithmically?}
\item[Q2.] \emph{What are the fundamental limits on policy evaluation imposed by the Huber-contaminated reward model?}
\end{itemize}

\textbf{Technical Challenges.} As it turns out, answering the above questions is quite non-trivial due to several technical challenges that we outline below.

$\bullet$ \textbf{Noisy Heavy-Tailed Rewards.} Even in the absence of corruption, note that our observation model allows the rewards to be noisy/random. Furthermore, we do not assume that the true reward distributions are sub-Gaussian; instead, we only require them to have finite mean and variance. This is unlike recent works on TD learning~\citep{bhandari_finite, srikant}, where the rewards are assumed to be deterministic (conditioned on the state). The possibility of \emph{heavy-tailed} uncorrupted rewards, in tandem with the lack of knowledge of the reward function $R_{\mu}$, \emph{significantly complicates the learner's task of distinguishing between clean and corrupted data}. 

$\bullet$ \textbf{Temporal Correlations.} In the standard robust statistics literature~\citep{huber, lai, chenrobust}, the inliers (i.e., clean data) are generated i.i.d from an unknown distribution. However, the observations in our setting are all part of one single Markovian trajectory and, as such, exhibit \textit{temporal correlations}. Even in the non-adversarial setting, obtaining finite-time results under Markovian data is known to be highly non-trivial. \emph{Moreover, contending with data that is simultaneously temporally correlated and adversarially contaminated has not been previously explored before}, thereby requiring the development of novel algorithmic and analysis techniques - this is one of the most challenging aspects of our problem. 

Additionally, the function approximation setting we consider here is much harder to tackle relative to a tabular setting. Despite the complex interplay between the challenges listed above, we will provide a precise finite-time analysis of a robust variant of \texttt{TD}(0) to be developed later in Section~\ref{sec:robustTD}. In the next section, we justify the need for such a development. 

\section{Motivation: Vulnerability of \texttt{TD}(0)}
\label{sec:vulnerability}
\vspace{-2mm}
The purpose of this section is to formally establish that the basic \texttt{TD}(0) algorithm is not robust to reward-poisoning attacks. To proceed, we make the following assumption that is standard in the analysis of RL algorithms~\citep{tsitsiklisroy, bhandari_finite,srikant}. 

\begin{assumption} \label{ass:MC} The Markov chain induced by the policy $\mu$ is aperiodic and irreducible. 
\end{assumption}
Under the above assumption, when the rewards are uncorrupted, \texttt{TD}(0) converges to $\theta^* = - \bar{A}^{-1} \bar{b}$, where $\bar{A} = \Phi^{\top} D \left(\gamma P_{\mu} - I \right) \Phi$, and $\bar{b} =  \Phi^{\top} D R_{\mu}$ are the steady-state versions of $A_t$ and $b_t$, respectively, and ${R}_{\mu} \in \mathbb{R}^{m}$ is a reward vector stacking up the mean rewards for the different states~\citep{tsitsiklisroy}. To isolate the effect of Huber contamination, it suffices to consider a \emph{noiseless} reward model where, whenever in state $s \in \mc{S}$, with probability $1-\epsilon$, the learner observes the true mean reward $R_{\mu}(s).$\footnote{We emphasize here that the assumption of noiseless rewards is only made for this motivating section. In the sequel, when we consider the problem of defending against reward contamination, we will  work under the more general heavy-tailed noisy reward model described in Section~\ref{sec:model}.} To capture corruption, consider the following error model: for each state $s\in\mc{S}$, whenever in state $s$, with probability $\epsilon$, the learner receives a bounded, deterministic signal $C(s)$. Let $C \in \mathbb{R}^{m}$ be the corrupted reward vector that stacks up $C(s)$ for each $s \in \mc{S}$. We then have the following simple result that characterizes the limit point of \texttt{TD}(0) under the above reward contamination model; see Appendix~\ref{app:proofThm1} for its proof. 

\begin{theorem} (\textbf{Vulnerability of \texttt{TD}(0)}) 
\label{thm:vulnerability}
Suppose Assumption~\ref{ass:MC} holds, and the step-size sequence $\{\alpha_t\}$ of \texttt{TD}(0) is chosen to satisfy $\sum_{t=0}^{\infty} \alpha_t = \infty$ and $\sum_{t=0}^{\infty} \alpha^2_t < \infty$. Then, under the Huber-contaminated reward model described above, the iterates of vanilla \texttt{TD}(0) converge with probability $1$ to $\tilde{\theta}^*$, where 
\begin{equation}
\tilde{\theta}^* = (1-\epsilon) \theta^* + \epsilon (- \bar{A}^{-1} \Phi^{\top} D C).
\label{eqn:corruptedTDfixedpt}
\end{equation}
Furthermore, for every point $w \in \mathbb{R}^K$, there exists a corresponding choice of corrupted reward vector $C_w$ that ensures $\tilde{\theta}^* = w.$
\end{theorem}

\textbf{Discussion.} The above result reveals that under the standard conditions for the convergence of $\texttt{TD}(0)$, the limit point of \texttt{TD}(0) with reward contamination is a convex combination of the true solution $\theta^*$ and a vector $- \bar{A}^{-1} \Phi^{\top} D C$ that can be controlled by the adversary by tuning $C$. The result also tells us that by carefully designing $C$, the adversary can cause the perturbed limit point $\tilde{\theta}^*$ to be any arbitrary point in $\mathbb{R}^K$. 

\begin{wrapfigure}[14]{R}[0pt]{0.3\textwidth}
 \vspace{-20pt}
 \begin{center}
  \includegraphics[scale=0.3]{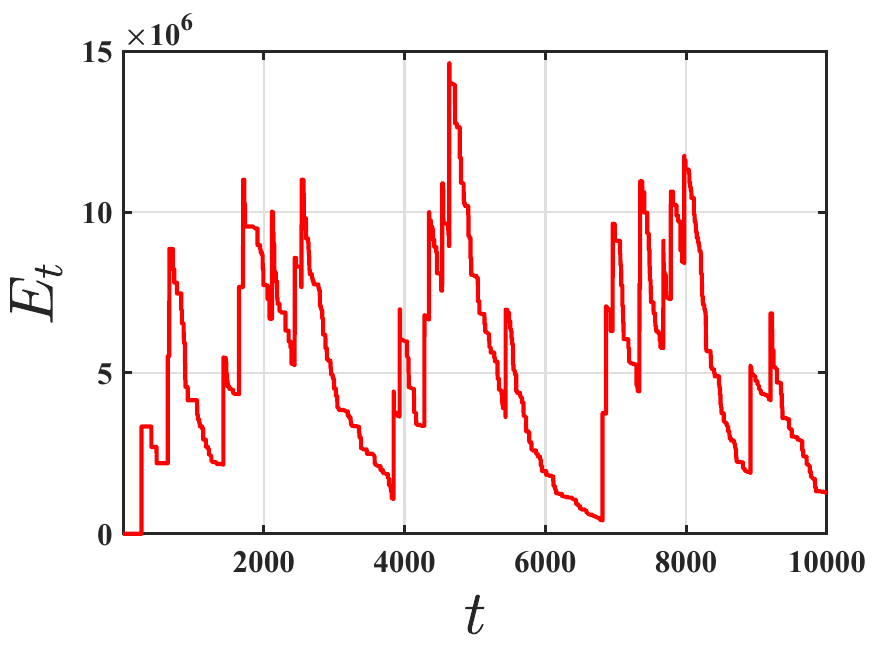} 
\end{center}
\caption{Plot of mean-square error $E_t$ showing the effect of reward corruption on \texttt{TD}(0), with corruption probability $\epsilon =0.001$.} 
\label{fig:vulnerability}
  \vspace{-3mm}
\end{wrapfigure}

The key takeaway from the above result is that even when the corruption fraction $\epsilon$ is small (but non-zero), the adversary can cause the true limit point $\theta^*$ of \texttt{TD}(0) to get perturbed to \emph{any} point in $\mathbb{R}^K$. To complement this result, we illustrate in Fig.~\ref{fig:vulnerability} a scenario where, even when the corruption fraction $\epsilon$ is merely $0.001$, the mean-square error of \texttt{TD}(0) can be large. 
Motivated by the finding from Theorem~\ref{thm:vulnerability}, we will systematically design a robust version of \texttt{TD}(0) in the sequel. As our first step towards this goal, we develop a robust univariate mean estimator for time-correlated data in the next section.  
\vspace{-3mm}
\section{Robust Markovian Mean Estimation}
\label{sec:RobustMeanEst}
\vspace{-3 mm}
This section investigates the problem of robust mean estimation given \emph{dependent} data samples generated by a Markov chain and corrupted as per the Huber attack model. The resulting developments will provide the key technical tools needed to tackle the robust TD learning problem. To provide context, suppose we are given $N$ i.i.d. real-valued random variables $X_1, X_2, \ldots, X_N$ with finite mean $\mathbb{E}[X_1]= \bar{X}$, and finite variance $\rho^2$. The goal is to construct an estimator $\hat{X}_N = \hat{X}_N(X_1, \ldots, X_N)$, i.e., a measurable function of $X_1, \ldots, X_N$, that provides a high-probability estimate of the mean $\bar{X}$. It is well known that the empirical mean fails to provide optimal error rates in this setting~\citep{lugosiHT}. Instead, a simple yet powerful estimator known as the Median-of-Means (\texttt{MoM}) estimator yields the following optimal rate: given any $\delta \in (0,1)$,  with probability (w.p.) at least $1-\delta$, the \texttt{MoM} estimate $\hat{X}_N$ satisfies: $|\hat{X}_N - \bar{X}| \leq O\left(\rho \sqrt{ {\log(1/\delta)}/{N}}\right)$~\citep{lugosiHT}. We depart from this setting by considering the \emph{dependent-data} observation model below. 

\textbf{The Setting.} Consider an ergodic, time-homogeneous, and stationary Markov chain $X_1, X_2, \ldots$ with stationary distribution $\pi$. The Markov chain takes values over a finite state space $\mc{X}.$ Let $F:\mc{X} \mapsto \mathbb{R}$ be a bounded function such that $|F(x)| \leq \psi, \forall x \in \mc{X}$, and define $\bar{f} \triangleq \mathbb{E}_{X \sim \pi}[F(X)]$. Now consider a noisy observation model characterized by a set of conditional distributions $\{\mc{D}(\cdot|x), x \in \mc{X}\}$, such that whenever in state $x$, the learner gets to observe a random sample $f(x)$ with mean $\mathbb{E}_{f(x) \sim \mc{D}(\cdot|x)}[f(x)] = F(x)$ and finite variance $\rho^2.$ Our goal is to obtain a high-probability estimate of $\bar{f}= \mathbb{E}_{X_i \sim \pi, f(X_i) \sim \mc{D}(\cdot|X_i)}[f(X_i)], i \in [N]$ for the following observation model. 

\textbf{Noisy and Corrupted Markovian Data:} The learner observes $N$ noisy and Huber-contaminated samples $\tilde{f}_1, \ldots, \tilde{f}_N$, where $\tilde{f}_i \sim (1-\epsilon)\mc{D}(\cdot| X_i) + \epsilon \mc{Q}, i \in [N]$, and $\mc{Q}$ is an unknown and unconstrained \emph{adversarial} error distribution.   

\emph{For the model above, can we expect mean-estimation bounds similar to those known under i.i.d data?} In what follows, we will provide an answer in the affirmative by developing a simple variant of the \texttt{MoM} estimator. Our estimator and its guarantees will depend on the notion of a mixing time $\tau_{mix}$. To define this object, let $\mathbb{P}(X_{t+1} \in \cdot| X_1 =X) $ denote the conditional distribution of $X_{t+1}$ given that $X_1 = X.$ Next, following~\cite{dorfman}, define $d_{mix}(t) \triangleq \sup_{X \in \mc{X}} D_{TV}\left(\mathbb{P}(X_{t+1} \in \cdot| X_1 =X), \pi \right)$, where $D_{TV}(\mc{P}_1, \mc{P}_2)$ denotes the total variation distance between two probability measures $\mc{P}_1$ and $\mc{P}_2$. We then define the mixing time $\tau_{mix}$ as follows: 
\begin{equation}
\tau_{mix}(\eta) \triangleq \inf\{t: d_{mix}(t) \leq \eta\}, \, \tau_{mix} \triangleq \tau_{mix}(1/4). 
\label{eqn:mix1}
\end{equation}
With the above model and notation in place, we are now ready to describe our estimator called \texttt{RUMEM}: Robust Univariate Mean Estimator for Markovian Data. 

\textbf{Description of \texttt{RUMEM}}: The algorithm takes as input a data set $\mc{S}=\{\tilde{f}_1, \ldots, \tilde{f}_N\}$ generated as per the noisy and corrupted Markovian data model, a confidence parameter $\delta$, the mixing time $\tau_{mix}$, and the corruption fraction $\epsilon.$ It then performs the following operations. 

\textbf{1) Subsampling.} The first step is to create a subsampled set $\mc{S}_{sub}=\{\tilde{f}_1, \tilde{f}_{\tau+1}, \ldots, \tilde{f}_{(n-1)\tau+1}\}$ by selecting every $\tau$-th element from the original data set $\mc{S}$; here, $n=\floor{{(N-1)}/{\tau}}+1$. The rationale behind this step is to create a data set comprising approximately independent samples by choosing the parameter $\tau$ appropriately; we will specify this choice in the statement of Theorem~\ref{theorem:RobustMean}. 

\textbf{2) Partitioning the Subsampled Set}: Next, the subsampled set $\mc{S}_{sub}$ is split into $L$ equal buckets denoted $\{\mc{B}_1, \ldots, \mc{B}_L\}$, where each bucket $\mc{B}_i$ has size given by $\floor{{n}/{L}}$.

\textbf{3) Calculating the Median-of-Means Estimate}: Let the mean of the samples in the $i$-th bucket $\mc{B}_i$ be denoted $\hat{\mu}_i$. The algorithm then returns $\hat{\mu}=\texttt{Median}\{\hat{\mu}_1, \ldots, \hat{\mu}_L\}.$ Other than the sub-sampling step, the rest of \texttt{RUMEM} is  the same as a standard \texttt{MoM} estimator; \emph{as such, we claim no novelty in the design of \texttt{RUMEM}. Instead, our main contribution here lies in the analysis of \texttt{RUMEM} under heavy-tailed, corrupted, and Markovian data.} It is this analysis that guides the choice of the two main parameters in the algorithm, namely the sub-sampling gap $\tau$ and the number of buckets $L$. The main result of this section is as follows. 

\begin{theorem}\label{theorem:RobustMean} (\textbf{Performance of \texttt{RUMEM}}) Consider the Huber-contaminated Markov data model described in this section with corruption fraction $\epsilon$. Let \texttt{RUMEM} be run on this data set with parameters chosen as follows:
\begin{equation}\begin{aligned}
&\tau = \ceil{\log_2\left(6N/\delta\right)\tau_{mix}}; \hspace{1mm} \epsilon^\prime =\epsilon + \frac{32}{3n}\log(24/\delta); \\
&L = \ceil{12\epsilon^\prime n + \frac{256}{7}\log\left(\frac{N}{\delta}\right)}.\footnote{Unless specified, we will use $\log(\cdot)$ to denote the natural logarithm.}\end{aligned}
\label{eqn:paramchoices}
\end{equation}
There exists a universal constant $C \geq 1$ such that given any $\delta \in (0,1)$, if $N \ge \max\{2, 4L\tau\}$, then with probability at least $1-\delta$, the output $\hat{\mu}$ of \texttt{RUMEM} satisfies 
\begin{equation}
\left\lvert \hat{\mu} - \bar{f} \right\rvert \leq C \max\{\psi, \rho\} \left( \sqrt{\epsilon} + \sqrt{\frac{\tau}{N}\log\left(\frac{N}{\delta}\right)} \right).
\label{eqn:RUMEMbnd}
\end{equation} 
\end{theorem}
\textbf{Discussion.} To appreciate the bound in Theorem~\ref{theorem:RobustMean}, consider first the case when the data is generated in an i.i.d. manner; here, the subsampling gap $\tau$ is simply $1$. With $\tau=1$, our bound in Eq.~\eqref{eqn:RUMEMbnd} is consistent with that known for robust univariate mean estimation under i.i.d. data with a trimmed mean estimator; see~\cite[Theorem 1]{lugosi}. While the \texttt{MoM} estimator is known to be robust against heavy-tailed i.i.d. noise, the fact that it is also robust to adversarial noise is new. Under Markov data, our bound gets inflated by a factor of $\sqrt{\tau}$. This is also consistent with mean estimation results under Markov data since one essentially has $N/\tau$ ``effective" samples~\citep{nagaraj, dorfman}. The significance of Theorem~\ref{theorem:RobustMean} lies in \textbf{providing the first guarantees of robust mean estimation under both Markovian and adversarial data.} This result might be of independent interest to robust statistics, and we conjecture that it will find use in dealing with outliers in time-series data (beyond our TD learning setting). The proof of Theorem~\ref{theorem:RobustMean} is provided in Appendix~\ref{app:proofrobustmeanest}. One subtle aspect of the proof is that it needs to account for the fact that the number of ``good" uncorrupted buckets is a random object. The other key ingredients in the analysis involve carefully exploiting the geometric mixing property of the underlying Markov chain along with a coupling idea in the recent paper of~\cite{dorfman}. 
\begin{algorithm}[t]
\caption{\texttt{Robust-TD} Algorithm}
\label{algo:algoRobustTD}  
 \begin{algorithmic}[1]
\State \textbf{Input:} Policy to be evaluated $\mu$, initial estimate $\theta_0 \in \mathbb{R}^{K}$, corruption fraction $\epsilon$, total number of iterations $T$, and burn-in time $\bar{T}$. 
\For {$t=0,1, \ldots$} 
\State Play $\mu(s_t)$ and observe tuple $\tilde{X}_t = (s_t, s_{t+1}, \tilde{r}_t)$, where $\tilde{r}_t \sim (1-\epsilon) \mc{D}_\mu(\cdot|s_t) + \epsilon \mc{Q}$. 
\State If $t \leq \bar{T}$,  then maintain $\theta_t = \theta_0$. 
\If{$ t > \bar{T}$} 
\State \hspace{-12mm} Set 
$[\hat{b}_t]_i = \textrm{\texttt{RUMEM}}\left( \{y_{i,k}\}_{0 \leq k \leq t}; \delta = 1/(KT^2)\right)$ with $y_{i,k} = [\phi(s_k)]_i \tilde{r}_k$ to compute  $\hat{b}_t$.
\State Compute threshold $G_t$ as per Eq.~\eqref{eqn:threshold} and set $\sigma_1 = \max\{1, \bar{r}, \rho\}$. 
\State If $\lVert\hat{b}_t\rVert_2 > G_t + \sigma_1$, then set:  $\hat{b}_t \gets 0$. 
\State Compute Robust TD direction
$\tilde{g}_t(\theta_t)$:
\begin{equation*}
    \tilde{g}_t(\theta_t)=A_t \theta_t + \hat{b}_t, \, \medmath{A_t = \gamma \phi(s_t) \phi^{\top}(s_{t+1}) - \phi(s_t) \phi^{\top}(s_t)}. 
\label{eqn:robustTDdirection}
\end{equation*}
\State Update parameter: $\theta_{t+1}=\theta_t + \alpha \tilde{g}_t(\theta_t)$.  
\EndIf
\EndFor
\end{algorithmic}
 \end{algorithm}

\section{Robust TD Learning Algorithm}
\label{sec:robustTD}
\vspace{-3mm}
In this section, we develop our proposed algorithm called \texttt{Robust-TD}; the steps of our method are outlined as Algorithm~\ref{algo:algoRobustTD}. As we shall see later, despite the presence of Huber-contaminated rewards, \texttt{Robust-TD} yields guarantees that are consistent with those provided by TD in the absence of attacks, up to a small \emph{unavoidable}  $O(\sqrt{\epsilon})$ term that captures the price of corruption. Our approach rests on two new ideas: (i) Using historical information of rewards along with the \texttt{MoM} estimator in Section~\ref{sec:RobustMeanEst} to construct robust TD update directions; and (ii) a carefully designed dynamic thresholding scheme to account for rare (i.e., low-probability) events. We describe these ideas below. 

$\bullet$ \textbf{Constructing Robust TD update directions.} To build intuition, let us start by noting from Eq.~\eqref{eqn:TD(0)update} that the vanilla \texttt{TD}(0) update direction (without corruption) can be expressed in affine form: $g_t(\theta) = A_t \theta + b_t$, where $A_t =\gamma \phi(s_t) \phi^{\top}(s_{t+1}) - \phi(s_t) \phi^{\top}(s_t)$ and $b_t= \phi(s_t) r_t$. The main observation here is that the rewards only affect the term $b_t$; as such, our goal will be to obtain a robust estimate of this object. Due to Assumption~\ref{ass:MC}, $b_t$ will eventually approach its stationary value $\bar{b}= \mathbb{E}_{s_t \sim \pi, r(s_t) \sim \mathcal{D}_{\mu}(\cdot|s_t)} [b_t]= \sum_{s \in \mc{S}} \pi(s) \phi(s) R_\mu(s)$. We would thus ideally like our robust estimate to be ``close" to $\bar{b}$. There are a few subtleties here. To explain them, let us consider a couple of candidate strategies. Given the structure of $\bar{b}$, one natural idea could be to use all prior observations collected for each state $s \in \mc{S}$ to construct estimates of $R_{\mu}(s)$ and $\pi(s)$ individually. However, this would require maintaining vectors of dimension equal to the size $|\mc{S}|$ of the state space, defeating the purpose of function approximation. Since only the rewards are corrupted, yet another strategy could be to apply the \texttt{RUMEM} estimator from Section~\ref{sec:RobustMeanEst} to the set of reward observations $\{\tilde{r}_k\}_{0 \leq k \leq t}$ collected up to time $t$. While this will yield a robust estimate of $\sum_{s \in \mc{S}} \pi(s) R_\mu(s)$, our goal instead is to get an estimate of $\sum_{s \in \mc{S}} \pi(s) \phi(s) R_\mu(s)$. The main message here is that some care needs to be taken while devising an approach for estimating $\bar{b}$. 

Our approach is to maintain component-wise estimates of $\bar{b}$. To that end, let $[\bar{b}]_i$ and $[\phi(s)]_i$ represent the $i$-th components of $\bar{b}$ and $\phi(s)$, respectively; here, $i \in [K]$. Moreover, let $\hat{b}_t$ denote the estimate of $\bar{b}$ at time $t$. Then, the $i$-th component of $\hat{b}_t$ is constructed by applying the \texttt{RUMEM} estimator in Section~\ref{sec:RobustMeanEst} to the data set $\{y_{i,k}\}_{0 \leq k \leq t}$ with a confidence parameter $\delta = 1/(KT^2)$, where $y_{i,k} = [\phi(s_k)]_i \tilde{r}_k$, and $T$ is the total number of iterations. We represent this operation succinctly as $[\hat{b}_t]_i = \textrm{\texttt{RUMEM}}\left( \{y_{i,k}\}_{0 \leq k \leq t}; \delta = 1/(KT^2)\right)$; see line 6 of Algo.~\ref{algo:algoRobustTD}. The intuition here is simple: if a sample at time $k$ is uncorrupted (i.e., $\tilde{r}_k = r(s_k)$), then we have
$ [\bar{b}]_i = \mathbb{E}_{s_k \sim \pi, r(s_k) \sim \mc{D}_{\mu}(\cdot|s_k)}[[\phi(s_k)]_i r(s_k)].$ In words, if the underlying Markov chain is stationary, then each uncorrupted sample $y_{i,k}$ provides an unbiased estimate of $[\bar{b}]_i$. Notably, however, these samples are \emph{not independent} - this is precisely why we need to appeal to the results from  Section~\ref{sec:RobustMeanEst}. 

$\bullet$ \textbf{Thresholding mechanism.} We now explain that the design of $\hat{b}_t$ as described above is not enough for our purpose. From Theorem~\ref{theorem:RobustMean}, note that the guarantees afforded by \texttt{RUMEM} do not hold deterministically, rather only with high probability. Since we seek to obtain mean-square error bounds, \emph{we need to thus provide an additional layer of safety for low-probability events on which the guarantees from \texttt{RUMEM} do not hold}. Accordingly, if $\Vert \hat{b}_t \Vert_2$ exceeds a carefully designed threshold $G_t + \sigma_1$, we reset $\hat{b}_t$ to $0$. Here, $\forall t \geq \bar{T}$, 
\begin{equation}
G_t \triangleq C \sqrt{K}  \sigma_1 \left( \sqrt{\epsilon} + 2 \log(12KT^3)\sqrt{\frac{\tau_{mix}}{t}} \right), 
\label{eqn:threshold}
\end{equation}
where $\sigma_1 = \max\{1, \bar{r}, \rho\}$, $C$ is as in Eq.~\eqref{eqn:RUMEMbnd} of Theorem~\ref{theorem:RobustMean}, $K << |\mc{S}|$ is the number of feature vectors, $\bar{T}$ is an initial burn-in time, and $\tau_{mix}=\tau_{mix}(1/4)$ is as in Eq.~\eqref{eqn:mix1}, and corresponds to the mixing-time of the Markov chain induced by the policy $\mu$. The design of the thresholding radius $G_t + \sigma_1$ is very delicate: if the radius is too loose, it may lead to sub-optimal bounds; if it is too tight, we might reset $\hat{b}_t$ to $0$ too often unnecessarily, leading again to vacuous bounds. Our analysis in Appendix~\ref{app:Mainresultproof} reveals that if $G_t$ is designed as per Eq.~\eqref{eqn:threshold}, then the resetting operation in line 8 of Algorithm~\ref{algo:algoRobustTD} will get bypassed with high probability, and $\hat{b}_t$ will remain the output of the \texttt{RUMEM} operation in line 6. In other words, the resetting operation will take place only under extreme (low-probability) events, exactly as desired. Lemma~\ref{lemm:rob_est_main} in Section~\ref{sec:results} shows that the \texttt{RUMEM}-based estimation scheme in conjunction with the thresholding mechanism described above yields an accurate estimate of $\bar{b}$. Finally, the iterates are updated only after an initial burn-in time $\bar{T}$ that is logarithmic in $T$; the exact form of $\bar{T}$ will be specified in Section~\ref{sec:results}. This ensures enough data samples have been collected for the guarantees in Section~\ref{sec:RobustMeanEst} to kick in. 

\section{Main Results}
\vspace{-3mm} 
\label{sec:results}
The goal of this section is to state and discuss our main results: (i) a finite-time bound for \texttt{Robust-TD}, and (ii) a minimax lower-bound establishing the near-optimality of the guarantee from \texttt{Robust-TD}. To do so, we will need to introduce a bit of notation and terminology. Let $\Sigma = \Phi^\top D \Phi$. Since $\Phi$ is full column rank, $\Sigma$ is full rank with a strictly positive smallest eigenvalue $\omega <1$. Next, recall from Section~\ref{sec:robustTD} that the \texttt{TD}(0) update direction can be expressed as $g_t(\theta)=A_t(X_t) \theta + b_t(X_t)$. Also, recall the steady-state version of $A_t$ denoted by $\bar{A}=\mathbb{E}_{s_t \sim \pi, s_{t+1} \sim P_{\mu}(\cdot| s_t)}[A_t(X_t)]$. Let us now introduce the following definition of mixing time, which will play a key role in our analysis. 

\begin{definition} \label{def:mix} 
Define
$\tau'_{mix}(\eta) \triangleq \min\{t\geq1: \Vert \mathbb{E}\left[A_k(X_k)|X_0\right]-\bar{A}\Vert_2 \leq \eta, \forall k \geq t, \forall X_0\}.$ 
\end{definition}

Assumption~\ref{ass:MC} implies that the total variation distance between the conditional distribution $\mathbb{P}\left(s_t=\cdot|s_0=s\right)$ and the stationary distribution $\pi$ decays geometrically fast for all $t \geq 0$, regardless of the initial state $s\in\mc{S}$~\citep{levin2017markov}. As a result of this geometric mixing of the Markov chain, one can show that $\tau'_{mix}(\eta)$ in Definition \ref{def:mix} is $O\left(\log(1/\eta)\right)$~\citep{chenQ}. For our purpose, we set $\tau' = \tau'_{mix}(\alpha)$, where $\alpha$ is the step-size. Define $\bar{\tau}_{mix} \triangleq \max\{\tau', \tau_{mix}\}$, $d_t \triangleq \Vert \theta_t - \theta^* \Vert^2_2$, and $\sigma \triangleq \max\{\Vert \theta^* \Vert_2, \Vert \theta_0 \Vert_2, \sigma_1\}$, where recall that $\sigma_1 =\max\{1, \bar{r}, \rho\}.$ Our main result is then as follows.
\begin{theorem} (\textbf{Performance of \texttt{Robust-TD}}) 
\label{thm:main_result} 
Suppose Assumption~\ref{ass:MC} holds, and the initial distribution of $s_0$ is the steady-state distribution $\pi$. Let $G=K/(\omega^2 (1-\gamma)^2)$. There exist universal constants $c_1, c_2 \geq 1$ such that if the step-size $\alpha$, the burn-in time $\bar{T}$, and the number of iterations $T$ are chosen as follows:
\begin{equation}
    \begin{aligned}
        \alpha &= \frac{4}{\omega (1-\gamma)} \frac{\log(T)}{T}, \bar{T} = \ceil{c_1 \tau_{mix} \log^2(KT)} \, , \\
        T &\geq \max\left\{ \bar{T}+\tau', \frac{ c_2 \tau' \log(T)}{\omega^2 (1-\gamma)^2} \right\},
    \end{aligned}
\end{equation}
then \texttt{Robust-TD} guarantees:
\begin{equation}
\mathbb{E}[d_T] \leq \tilde{O}\left( \frac{\bar{\tau}_{mix } \sigma^2 G}{T} \right) + O\left(\epsilon \sigma^2_1 G\right). 
\label{eqn:main_bound}
\end{equation}
\end{theorem}

\textbf{Lower Bound Analysis.} From Theorem~\ref{thm:main_result}, we infer that despite adversarial contamination, the iterates generated by \texttt{Robust-TD} converge (in the mean-square sense) at a rate of $\tilde{O}(1/T)$ to a small ball of radius $O(\epsilon)$ around the optimal parameter $\theta^*$. Our next goal is to prove an information-theoretic lower bound to establish that the $O(\epsilon)$ additive error \emph{cannot be avoided, in general}. To do so, it suffices to consider a simpler i.i.d observation model where at each time-step $t$, $s_t$ is sampled independently from the stationary distribution $\pi$, and $s_{t+1}$ from ${P}_{\mu}(\cdot|s_t).$ We also consider a simpler tabular setting where the feature matrix $\Phi$ is the identity matrix of dimension $|\mc{S}|.$ Next, we use $\mc{M}(\epsilon, \rho, \mc{Q})$ to represent the set of all MRPs with finite state and action spaces and bounded mean rewards, where the reward random variable $\tilde{r}(s)$ is sampled as $\tilde{r}(s) \sim (1-\epsilon) \mc{D}_{\mu}(\cdot|s) + \epsilon \mc{Q}$, and the noise distribution $\mc{D}_{\mu}(\cdot|s)$ has variance at most $\rho^2$. We will use the shorthand $V \in \mc{M}(\epsilon, \rho, \mc{Q})$ to mean that $V$ is the true value function associated with some MRP in the set $\mc{M}(\epsilon, \rho, \mc{Q})$. Now, suppose the learner is presented with $T$ independent samples $\tilde{X}_1, \ldots, \tilde{X}_T$, where $\tilde{X}_t = (s_t, s_{t+1}, \tilde{r}(s_t)), t \in [T].$ An estimator $\hat{V}_T$ is some measurable function of these $T$ samples. We then have the following \emph{fundamental} lower bound.  

\begin{theorem} 
\label{thm:lowerbnd}
(\textbf{Lower Bound}) There exists a universal constant $\tilde{c} >0$ such that
$$ \inf_{\hat{V}_T} \sup_{V \in \mc{M}(\epsilon, \rho, \mc{Q})} \mathbb{P}\left( \Vert \hat{V}_T - V \Vert_2 \geq \frac{\tilde{c} \rho \sqrt{\epsilon}}{(1-\gamma)}\right) \geq \frac{1}{2}.$$ 
\end{theorem}

Before providing proof sketches for our main results, several remarks are in order. 

\textbf{Discussion.}  To put our result in Theorem~\ref{thm:main_result} into perspective, let us note that in the absence of corruption, i.e., when $\epsilon =0$, our convergence bound in Eq.~\eqref{eqn:main_bound} is consistent - up to log factors - with prior results on \texttt{TD}(0) with linear function approximation~\citep{bhandari_finite, srikant}. In particular, the dependence of the first term in the R.H.S. of Eq.~\eqref{eqn:main_bound} on each of the parameters $\bar{\tau}_{mix}, \omega, \gamma,$ and $T$ match those for vanilla \texttt{TD}(0) in \cite[Theorem 3]{bhandari_finite}. 

When $\epsilon \neq 0$, the R.H.S. of Eq.~\eqref{eqn:main_bound} features an additive $O(\epsilon)$ term. At this stage, it is instructive to compare this term with the analogous $O(\epsilon)$ term in Eq.~\eqref{eqn:corruptedTDfixedpt}. Crucially, with the basic \texttt{TD}(0) algorithm, the $O(\epsilon)$ term is affected by the \textbf{magnitude} of the attack corruption through the corruption vector $C$ (see~\eqref{eqn:corruptedTDfixedpt}). In contrast, with \texttt{Robust-TD}, the $O(\epsilon)$ term in the mean square error is \emph{completely unaffected by the magnitude of the attack inputs} and depends only on instance-dependent parameters. Specifically, the $O(\epsilon)$ term in Eq.~\eqref{eqn:main_bound} is scaled by the ``variance" $\rho^2$ of our noisy observation model; recall here that $\sigma_1 = \max\{1,\bar{r}, \rho\}$. We note such a $O(\epsilon)$ term - inflated by the noise variance - 
has been proven to be unavoidable in general for robust mean estimation~\citep{chenrobust, lai, cheng, Dalal2}. Similar unavoidable terms that capture the price of adversarial contamination also show up for multi-armed bandits with reward corruptions~\citep{lykouris, guptaadv, kapoor}. \textbf{Our work complements these results, and its significance lies in providing the first provable guarantees of adversarial robustness for TD learning.} 

However, one might still ask: \emph{In the context of policy evaluation in RL, is the $O(\epsilon)$ term inevitable or simply an artifact of our analysis?} Theorem~\ref{thm:lowerbnd} settles this question by establishing that it is the former. Comparing the lower bound in Theorem~\ref{thm:lowerbnd} with the upper bound in Eq.~\eqref{eqn:main_bound}, we also infer that the dependencies on the noise variance $\rho$ and the discount factor $\gamma$ via the $(1-\gamma)^{-1}$ term, as they appear in the $O(\epsilon)$ term of Eq.~\eqref{eqn:main_bound}, are tight. 

Finally, observe that the corruption-affected $O(\epsilon)$ term in Eq.~\eqref{eqn:main_bound} is inflated by $(1/\omega^2)$, where $\omega >0$ is the smallest eigenvalue of the steady-state feature covariance matrix $\Sigma = \Phi^\top D \Phi$. To gain some intuition, suppose for the moment that $\Phi$ is the identity matrix of order $m$. In this case, $\omega$ is simply the smallest entry in the steady-state distribution vector $\pi$. A small value of $\omega$ implies that the corresponding state is visited infrequently, that is, there is a paucity of data for such a state. Intuitively, this should favor the adversary, and make it harder to get a reliable estimate of the value function corresponding to the state that gets visited least frequently. Our upper bound captures this intuitive phenomenon; however, at the moment, we do not have a lower bound to support the dependence on $1/(\omega^2).$ Aside from this shortcoming, Theorems~\ref{thm:main_result} and~\ref{thm:lowerbnd} collectively paint a fairly complete picture of the problem of TD learning with Huber-contaminated adversarial rewards. To corroborate our theory, we provide various experiments on synthetic data in Appendix~\ref{app:Sims}. 

We conclude this section by providing proof sketches for our main results. 

\textbf{Proof Sketch for Theorem~\ref{thm:main_result}}. There are two different bias terms that affect the learning dynamics of \texttt{Robust-TD}: the term $\langle \theta_t - \theta^*, (A_t - \bar{A}) \theta_t \rangle$, capturing the effect of Markov sampling; and the term $\langle \theta_t - \theta^*, \hat{b}_t - \bar{b} \rangle$, capturing the effect of adversarial perturbations. Our setting is particularly complicated because these two terms are \emph{coupled}: the adversarial bias term affects the iterate $\theta_t$, which shows up in the Markovian bias term. \emph{Controlling this coupling in a way that leads to near-optimal guarantees - as in Theorem~\ref{thm:main_result} - has not appeared in prior RL work}. The key new technical ingredient unique to our analysis is the following lemma; its proof exploits the bound in Theorem~\ref{theorem:RobustMean}.

\begin{lemma} 
\label{lemm:robest_main}
(\textbf{Adversarial Perturbation Bound}) 
\label{lemm:rob_est_main} Under the conditions in the statement of Theorem~\ref{thm:main_result},
the following is true for all $t \geq \bar{T}$: 
$$\mathbb{E}[ \Vert \hat{b}_t - \bar{b} \Vert^2_2] \leq O\left(\epsilon + \frac{\log^2(KT) \tau_{mix}}{t} \right) K \sigma^2_1.$$
\end{lemma}

While this lemma helps us control the effect of the adversarial bias term, we need additional work to handle the effects of Markovian bias and function approximation. The thresholding operation in line 8 of \texttt{Robust-TD} helps us in this regard. The rest of the analysis proceeds by carefully leveraging the mixing properties of the underlying Markov chain in tandem with the bounds on each of the bias terms; the detailed steps are provided in Appendix~\ref{app:Mainresultproof}.  Notably, using the ergodicity of the Markov chain to handle outliers in time-correlated data is novel to our analysis, and might be of broader interest to both RL and robust statistics. Note that the assumption that the initial state is distributed as per the steady-state distribution ensures that the resulting Markov chain is stationary - this is a standard assumption made in prior work~\citep{bhandari_finite} to simplify some of the expressions. 

\textbf{Proof Sketch for Theorem~\ref{thm:lowerbnd}.} The proof of this result, detailed in Appendix~\ref{app:lowerbnd}, relies on carefully constructing two different MRPs and associated attack distributions, such that (i) the value functions in the two MRPs differ in magnitude by $\Omega(\rho \sqrt{\epsilon}/ (1-\gamma))$; and (ii) the distributions of the samples in the two MRPs is indistinguishable to a learner. We then leverage ideas to prove minimax lower bounds~\cite[Chapter 15]{wainwright} from statistical learning theory. 
\vspace{-2mm}
\section{Conclusion and Future Work} 
\vspace{-2mm}
We conducted the first principled study of TD learning with linear function approximation under a Huber-contaminated reward model. We started by showing that the basic TD algorithm is vulnerable to reward poisoning. We then developed a robust TD algorithm by drawing on median-of-means estimators, and by constructing a novel dynamic thresholding scheme. By establishing nearly matching upper and lower bounds, we showed that our proposed approach is provably robust to adversarial reward contamination. As future work, we plan to generalize our algorithm and results to nonlinear stochastic approximation schemes such as Q-learning; some preliminary results in this regard are reported in~\cite{maity2024robust}. We also plan to consider other attack models and derive finer lower bounds that account for Markov sampling. 
\newpage 
\bibliography{refs}
\bibliographystyle{apalike}
\clearpage
\newpage
\section*{Checklist}

 \begin{enumerate}
 \item For all models and algorithms presented, check if you include:
 \begin{enumerate}
   \item A clear description of the mathematical setting, assumptions, algorithm, and/or model. Yes. 
   \item An analysis of the properties and complexity (time, space, sample size) of any algorithm. Yes. 
   \item (Optional) Anonymized source code, with specification of all dependencies, including external libraries. No. 
 \end{enumerate}

 \item For any theoretical claim, check if you include:
 \begin{enumerate}
   \item Statements of the full set of assumptions of all theoretical results. Yes.
   \item Complete proofs of all theoretical results. Yes, we provide complete proofs in the Supplementary Material, and a proof sketch of our main results in Section 5. 
   \item Clear explanations of any assumptions. Yes.     
 \end{enumerate}

 \item For all figures and tables that present empirical results, check if you include:
 \begin{enumerate}
   \item The code, data, and instructions needed to reproduce the main experimental results. Yes, while we do not provide a code, we provide all the details needed to reproduce our ``toy" simulations in the Appendix.  
   \item All the training details (e.g., data splits, hyperparameters, how they were chosen). Yes, these are provided in the Appendix.
         \item A clear definition of the specific measure or statistics and error bars (e.g., with respect to the random seed after running experiments multiple times). Yes.
         \item A description of the computing infrastructure used. (e.g., type of GPUs, internal cluster, or cloud provider). Yes. 
 \end{enumerate}

 \item If you are using existing assets (e.g., code, data, models) or curating/releasing new assets, check if you include:
 \begin{enumerate}
   \item Citations of the creator If your work uses existing assets. Not Applicable.
   \item The license information of the assets, if applicable. Not Applicable.
   \item New assets either in the supplemental material or as a URL, if applicable. Not Applicable.
   \item Information about consent from data providers/curators. Not Applicable.
   \item Discussion of sensible content if applicable, e.g., personally identifiable information or offensive content. Not Applicable.
 \end{enumerate}

 \item If you used crowdsourcing or conducted research with human subjects, check if you include:
 \begin{enumerate}
   \item The full text of instructions given to participants and screenshots. Not Applicable.
   \item Descriptions of potential participant risks, with links to Institutional Review Board (IRB) approvals if applicable. Not Applicable.
   \item The estimated hourly wage paid to participants and the total amount spent on participant compensation. Not Applicable.
 \end{enumerate}

 \end{enumerate}

\onecolumn
\appendix
\section{Additional Literature Review}
\label{app:add_lit}
In this section, we provide a more detailed discussion of the relevant threads of literature.

\begin{enumerate}
\item \textbf{Temporal Difference Learning and Stochastic Approximation.} The family of TD learning methods was introduced by Sutton in~\cite{sutton1988}. The initial analysis of this algorithm with linear function approximation was carried out in~\cite{tsitsiklisroy} by drawing on tools from the rich area of stochastic approximation theory~\citep{borkar, borkarode}. The nature of the analysis in these works was \emph{asymptotic}, i.e., no convergence rates were provided. The next set of results~\citep{korda,  lakshmi, dalal, narayanan, prashanth2021conc} on this topic did manage to provide finite-time rates for TD methods; however, these results were obtained assuming that the samples are drawn i.i.d. from the steady state distribution of the underlying Markov chain. The i.i.d. assumption was first relaxed in~\cite{bhandari_finite}, where the authors relied on a projection step in their analysis. An analysis without the projection step was then provided in~\cite{srikant}, and more recently by~\cite{mitra2024simple} based on a novel inductive proof technique. An interesting interpretation of the TD update direction was provided in~\cite{liuTD} by introducing the notion of ``gradient-splitting". We note that in~\cite{bhandari_finite, srikant, liuTD, mitra2024simple}, the authors characterize finite-time bounds under linear function approximation in terms of an $\ell_2$-error metric. Our setting is similar. Complementary to $\ell_2$-error bounds, the work of~\cite{pananjady} provides $\ell_{\infty}$ bounds for the least squares temporal difference learning (LSTD) algorithm for a tabular setting, under the assumption of a generative data/observation model. For a more textbook treatment of the subject, we refer the interested reader to~\cite{suttonRL, szepesvari}.

While TD learning with linear function approximation is an instance of linear stochastic approximation, the analysis of TD learning with neural network approximation has been recently carried out in~\cite{tian, cayci}. Finite-time analysis of general nonlinear stochastic approximation schemes (that subsume variants of Q-learning) can be found in~\cite{Waiwright, shah2018, Qu, chenQ, chen2022Auto, li2024q}.  

Each of the papers mentioned above studies the basic versions of the concerned algorithms, where updates are made using noisy versions of some true underlying operator. Our work analyzes the robustness of these algorithms to adversarial perturbations. On a related note, we mention here that other types of perturbations resulting from communication-induced challenges (e.g., delays and compression) have been explored recently in~\cite{mitra2023temporal, adibi2024stochastic, dal2024finite}. 

\item \textbf{Reward Contamination in Multi-Armed Bandits.} A large body of work has explored the effects of reward contamination on the performance of stochastic bandit problems, both for the unstructured multi-armed bandit (MAB) setting~\citep{junadv,liuadv, kapoor, lykouris,guptaadv}, and also for structured linear bandits~\citep{bogunovic1,garcelon,bogunovic2,he22}. The basic premise in these papers is that an adversary can modify the true stochastic reward/feedback on certain rounds; a corruption budget $C$ captures the total corruption injected by the adversary over the horizon $T$. In particular, the authors in~\cite{kapoor} study a Huber-contaminated reward model like us, where in each round, with probability $\eta$ (independently of the other rounds), the attacker can bias the reward seen by the learner. A fundamental lower bound of $\Omega(\eta T)$ on the regret is also established in~\cite{kapoor}. While our reward contamination model is directly inspired by the above line of work, \textbf{we emphasize that the stochastic approximation setting we study here fundamentally differs from the bandit problem}. As such, our algorithms and proof techniques are also different from the bandit literature. 

\item \textbf{Robust Statistics.} The study of computing different statistics (e.g., mean, variance, etc.) of a data set in the presence of outliers was pioneered by Huber~\citep{huber, huber2}. Since then, the field of robust statistics has significantly advanced, with more recent work focusing on computationally tractable algorithms in the high-dimensional setting~\citep{lai, chenrobust, minsker, cheng, lugosi, Dalal2}. Our paper builds on this rich line of work and uses it in the context of RL. As mentioned earlier, the standing assumption in the existing robust statistics papers is that the inliers are generated in an i.i.d. manner. We relax this assumption for robust univariate mean estimation, and show how bounds can be obtained with correlated data generated from a Markov chain. 
\end{enumerate}
\section{Useful Results}
\label{app:basic}
In this section, we will compile some known results and facts that will play a key role in our subsequent analysis. In what follows, unless otherwise stated, we will use $\Vert \cdot \Vert$ to refer to the standard Euclidean norm. 

To proceed, we remind the reader that $\bar{A}= \Phi^{\top} D \left(\gamma P_{\mu} - I \right) \Phi$ and $\bar{b} =  \Phi^{\top} D R_{\mu}$ are the steady-state versions of $A_t$ and $b_t$, respectively. Recall that the mean-path/steady-state \texttt{TD}(0) update direction is as follows:
$$ \bar{g}(\theta) = \bar{A}\theta + \bar{b}.$$
The next result from \cite{bhandari_finite} tells us that the steady-state direction $\bar{g}(\theta)$ drives the iterates towards the optimal parameter $\theta^*$. 
\begin{lemma}
\label{lemma:convex}
The following holds $\forall \theta \in \mathbb{R}^K$:
$$ \langle \theta^* - \theta, \bar{g}(\theta) \rangle \geq \omega (1-\gamma) \Vert \theta^* -\theta \Vert^2,$$ 
 where $\omega$ is the smallest eigenvalue of the matrix $\Sigma = \Phi^\top D \Phi$.
\end{lemma}
\noindent Under the assumptions on the feature matrix $\Phi$ in Section~\ref{sec:model}, and Assumption~\ref{ass:MC}, it is easy to see that $\Sigma$ is positive definite with $\omega \in (0,1).$ Next, thanks to feature-normalization, it is easy to establish bounds on the norms of $\bar{A}_t$ and $\bar{A}$~\citep{bhandari_finite, srikant}:
\begin{equation}
\Vert A_t \Vert \leq 2, \forall t \in \mathbb{N}, \, \Vert \bar{A} \Vert \leq 2.
\label{eqn:norms}
\end{equation}

We will use the above fact at several points in our analysis. In addition to the above results, we will require a couple of standard concentration tools that we list here to keep the paper self-contained. For reference, see~\cite{Concentration, chungconc}.

\begin{lemma}
 (\textbf{Hoeffding's Inequality}) 
\label{lemma:hoeffding}
If $X_1, X_2, \ldots, X_N$ are independent random variables with $\mathbb{P}(a \leq X_i \leq b) = 1$ and common mean $\mu$, then for any $\epsilon > 0$:
\begin{equation}
\mathbb{P}(|\bar{X}_N - \mu| > \epsilon) \leq 2 \exp\left\{\frac{-2N\epsilon^2}{(b-a)^2}\right\},
\end{equation}
where $\bar{X}_N = \frac{1}{N} \sum_{i=1}^{N} X_i$.
\end{lemma}

\begin{lemma}(\textbf{Bernstein's Inequality})\label{lemma:bernstein}
If $X_1, X_2, \ldots, X_N$ are independent random variables with $\mathbb{P}(|X_i| \leq c) = 1$ and common mean $\mu$, then for any $\epsilon > 0$:
\begin{equation}
\mathbb{P}(|\bar{X}_N - \mu| > \epsilon) \leq 2 \exp\left\{ -\frac{N\epsilon^2}{2\sigma^2 + \frac{2c\epsilon}{3}} \right\},
\end{equation}
where $\bar{X}_N = \frac{1}{N} \sum_{i=1}^{N} X_i$ and $\sigma^2 = \frac{1}{N}\sum_{i=1}^{N} \text{Var}(X_i)$.
\end{lemma}
\section{Proof of Theorem~\ref{thm:vulnerability}}
\label{app:proofThm1}
In this section, we provide a proof for Theorem~\ref{thm:vulnerability}. First, suppose the corruption fraction $\epsilon =0$. In this case, under the conditions in Theorem~\ref{thm:vulnerability}, it is well known~\citep{tsitsiklisroy}  that \texttt{TD}(0) converges with probability one to $\theta^*= - \bar{A}^{-1} \bar{b}$, where $\bar{A} = \Phi^{\top} D \left(\gamma P_{\mu} - I \right) \Phi$, and $\bar{b} =  \Phi^{\top} D R_{\mu}$. Notice that the rewards only affect the term $\bar{b}$. Now under the Huber-contaminated reward model described in Section~\ref{sec:vulnerability}, for each state $s\in \mc{S}$, whenever in $s$, with probability $(1-\epsilon)$, the learner observes the true mean reward $R_\mu(s)$, and with probability $\epsilon$, a bounded, deterministic signal $C(s).$ Thus, effectively, the mean of the observed reward random variable $\tilde{r}(s)$ in state $s$ is $\tilde{R}(s) \triangleq (1-\epsilon) {R}_{\mu}(s) + \epsilon C(s).$ Stacking up the individual components $\tilde{R}(s)$ into a perturbed reward vector $\tilde{R}$, let us define $\tilde{b} \triangleq \Phi^{\top} D \tilde{R}.$ Since everything else remains the same, \texttt{TD}(0) will now converge with probability one to 
\begin{equation}\begin{aligned}
&\tilde{\theta}^* = - \bar{A}^{-1} \tilde{b} = - \bar{A}^{-1} \Phi^{\top} D \tilde{R}\\
&= \underbrace{(1-\epsilon) (-\bar{A}^{-1} \Phi^{\top} D R_{\mu})}_{(1-\epsilon) \theta^*} + \epsilon (- \bar{A}^{-1} \Phi^{\top} D C),\end{aligned}
\label{eqn:perturbedLP}
\end{equation}
where we used the expression for $\theta^*$. This establishes the first part of the theorem. 

For the second part, suppose the adversary wishes $\tilde{\theta}^*$ to be some point $w \in \mathbb{R}^K$. We will show that this is possible by explicitly designing an appropriate corrupted reward vector $C_w$. In particular, let 
$$ C_w = \frac{1}{\epsilon} D^{-1} \Phi \left(\Phi^{\top} \Phi \right)^{-1} \bar{A} \left((1-\epsilon) \theta^* - w \right).$$
We note here that in light of Assumption~\ref{ass:MC}, $\pi(s) > 0, \forall s \in \mc{S}.$ Hence, $D^{-1}$ exists. Moreover, the fact that $\Phi$ is full column rank ensures the existence of $\left(\Phi^{\top} \Phi \right)^{-1}$. Plugging in our choice of $C_w$ into Eq.~\eqref{eqn:perturbedLP}, it is easy to see that $\tilde{\theta}^* =w.$ This completes the proof. 

\section{Performance Analysis for \texttt{RUMEM}: Proof of Theorem~\ref{theorem:RobustMean}}
\label{app:proofrobustmeanest}
In this section, we will prove Theorem~\ref{theorem:RobustMean}. Let us recall the setting quickly. First, $X_1, \ldots, X_N$ is a sequence of $N$ samples drawn from a stationary Markov chain with stationary distribution $\pi$. Thus, each $X_i$ is distributed as per $\pi$. Next, for each $i \in [N]$, with probability $(1-\epsilon)$, the learner observes a random variable $f(X_i)$ with mean $F(X_i)$ and variance at most $\rho^2$; and with probability $\epsilon$, an arbitrary object chosen by the adversary. The resulting sample is denoted $\tilde{f}_i$. 

Our analysis will rely on a coupling argument that is inspired by the recent work of~\cite{dorfman}. To apply this argument, consider the sub-sampled sequence $X_1, X_{\tau+1}, \ldots, X_{(n-1)\tau+1},$ where $n$ and $\tau$ are as in Section~\ref{sec:RobustMeanEst}. We couple this sequence with its i.i.d. counterpart $(X_{\texttt{I},1}, X_{\texttt{I},\tau+1}, \ldots, X_{\texttt{I}, (n-1)\tau+1}) \sim \pi^{\otimes n}$. Here, and henceforth throughout the proof, we will use the subscript $\texttt{I}$ to denote the i.i.d. counterpart of a Markov sample. The next result bounds the probability of the Markovian sub-sampled sequence being different from its i.i.d. counterpart.  

\begin{lemma} \cite[Lemma 3]{nagaraj}
\label{lemm:coupled}
Let $\mc{E}_1$ be an event where $(X_1, X_{\tau+1}, \ldots, X_{(n-1)\tau+1}) = (X_{\texttt{I},1}, X_{\texttt{I},\tau+1}, \ldots, X_{\texttt{I}, (n-1)\tau+1}).$ Then,
$$ \mathbb{P}(\mc{E}^c_1) \leq (n-1)d_{mix}(\tau),$$
\label{lemm:coupling}
where $d_{mix}(\cdot)$ is as defined as
$$ d_{mix}(t) \triangleq \sup_{X \in \mc{X}} D_{TV}\left(\mathbb{P}(X_{t+1} \in \cdot| X_1 =X), \pi \right).$$
\end{lemma}

We will call upon the above lemma at a later point in our analysis. For now, we split the proof into multiple steps. 

\textbf{Step 1. Bounding the number of corrupted samples.} The size of the sub-sampled set used in the \texttt{RUMEM} algorithm is $n$. Our first step is to control the number of corrupted samples in this sub-sampled set. To that end, consider an event $\mathcal{E}_2$, where the maximum number of corrupted samples in the sub-sampled set is $\frac{3\epsilon^\prime n}{2}$; recall here that
$$ \epsilon^\prime = \epsilon + \frac{32}{3n} \log\left(\frac{24}{\delta}\right).$$
Our immediate goal is to provide an upper bound on the complementary event $\mathcal{E}_2^c$. With this in mind, let $W_i$ be an indicator random variable, such that $W_i = 1$ if the $i^{th}$ sub-sample is corrupted, and 0 otherwise. Based on the Huber attack model, $\mathbb{E}[W_i] = \epsilon, \forall i \in [n]$.  Furthermore, $ \frac{1}{n}\sum_{i=1}^{n} \text{Var}(W_i) \le \epsilon$. Now observe: 
\begin{equation}
\begin{aligned}
    &\mathcal{E}_2^c =\Bigl\{ \sum_{i=1}^n W_i \ge \frac{3\epsilon^\prime n}{2} \Bigr\}\\
    & \hspace{3 mm}=\Bigl\{ \frac{1}{n}\sum_{i=1}^n W_i - \epsilon \ge \frac{3\epsilon^\prime}{2}-\epsilon \Bigr\}\\
    & \implies \Bigl\{ \frac{1}{n}\sum_{i=1}^n W_i - \epsilon \ge \frac{\epsilon^\prime}{2}\Bigr\},
\end{aligned}
\end{equation}
where in the last step, we used $\epsilon' > \epsilon.$ Invoking Bernstein's inequality (Lemma \ref{lemma:bernstein}) then yields \begin{equation}\label{eqn:bern_mom}
\mathbb{P}\left(\mathcal{E}_2^c\right) \le 2 e^{-\frac{3\epsilon^\prime n}{32}}. 
\end{equation} 
This completes step 1.  

\textbf{Step 2: Statistics of each non-contaminated sample.} Suppose a particular sample $i$ is non-contaminated, i.e., $\tilde{f}_i = f(X_i).$ We then have
\begin{equation}
\begin{aligned}
&\mathbb{E}_{X_i \sim \pi, f(X_i) \sim \mc{D}(\cdot|X_i)}[f(X_i)] \\
&= \sum_{x \in \mc{X}} \mathbb{E}_{f(X_i)\sim \mc{D}(\cdot| X_i)} [f(X_i)|X_i=x] \pi(x) \\
& = \sum_{x \in \mc{X}}  \mathbb{E}_{f(x)\sim \mc{D}(\cdot| x)}[ f(x) ] \pi(x)\\
& = \sum_{x \in \mc{X}}  F(x) \pi(x) \\
& = \bar{f}. 
\end{aligned}
\end{equation}

Now given that $|F(x)| \leq \psi, \forall x \in \mc{X}$, we have $|\bar{f}| \leq \sum_{x \in \mc{X}} |F(x)| \pi(x) \leq \psi.$ Using this,  we can bound the variance of a non-contaminated sample as follows:
\begin{equation}\label{eqn:variancefbar}
\begin{aligned}
&\mathbb{E}_{X_i \sim \pi, f(X_i) \sim \mc{D}(\cdot|X_i)}[f^2(X_i)] \\
&= \sum_{x \in \mc{X}} \mathbb{E}_{f(X_i)\sim \mc{D}(\cdot| X_i)} [f^2(X_i)|X_i=x] \pi(x) \\
& = \sum_{x \in \mc{X}}  \mathbb{E}_{f(x)\sim \mc{D}(\cdot| x)}[ f^2(x) ] \pi(x)\\
& \overset{(a)}\leq  \sum_{x \in \mc{X}} (\bar{f}^2 + \rho^2) \pi(x)\\
& \overset{(b)}\leq  \sum_{x \in \mc{X}} (\psi^2 + \rho^2) \pi(x)\\
& \leq \bar{\sigma}^2,
\end{aligned}
\end{equation}
where $\bar{\sigma}^2 \triangleq 2 (\max\{\psi, \rho\})^2.$ For $(a)$, we used the basic identity: $\text{Var}(X) = \mathbb{E}(X^2)-\mathbb{E}(X)^2$;  and for $(b)$, we used the fact that $\bar{f} \le \psi$ and that the variance of $f(x)$ is at most $\rho^2$.  

In light of Lemma~\ref{lemm:coupled}, we will first focus on assessing the performance of \texttt{RUMEM} on the i.i.d. sub-sampled sequence 
 $(X_{\texttt{I},1}, X_{\texttt{I},\tau+1}, \ldots, X_{\texttt{I}, (n-1)\tau+1})$. For this data set, let the mean of the $i$-th bucket $\mc{B}_i$ be denoted $\hat{\mu}_{\texttt{I}, i}$, and the final \texttt{MoM} estimate be denoted $\hat{\mu}_{\texttt{I}}  = \texttt{Median}\{\hat{\mu}_{\texttt{I},1}, \hat{\mu}_{\texttt{I},2}, \ldots, \hat{\mu}_{\texttt{I},L}\}$. 

\textbf{Step 3. Bound on performance for each non-contaminated bucket under i.i.d. data.} Consider a bucket $\mc{B}_i$ that only contains non-contaminated i.i.d. samples. Based on Step 2, each sample has mean $\bar{f}$ and variance at most $\bar{\sigma}^2.$ Using $M=\floor{n/L}$ to denote the number of samples in each bucket, we obtain the following bound using Chebyshev's inequality $\forall d > 0$:
    \begin{equation}\label{eqn:momiidstep1}
    \mathbb{P}\left(\hat{\mu}_{\texttt{I},i} \ge \bar{f}+\frac{d \bar{\sigma}}{\sqrt{\frac{N}{\tau}}}\right) \le \frac{\frac{\bar{\sigma}^2}{M}}{\frac{d^2  \bar{\sigma}^2}{\frac{N}{\tau}}} = \frac{N}{M\tau d^2}. 
    \end{equation}
To proceed, we will require a lower bound on the number of samples $M$ in each bucket. To that end, we start by noting that
$$ n=\floor{\frac{N-1}{\tau}}+1 \geq \frac{N-1}{\tau} \geq \frac{N}{2\tau},$$
where in the last step, we used $N \geq 2$. Next, using $N \geq 4 L \tau$ - as required in the statement of Theorem~\ref{theorem:RobustMean} - we obtain
$$ M = \floor{n/L} \geq \frac{n}{L} - 1 \geq \frac{N}{2L\tau} - 1 \geq \frac{N}{4 L \tau}.$$
Plugging the above bound back in Eq.~\eqref{eqn:momiidstep1}, and setting $d = 4\sqrt{L}$, we obtain
   \begin{equation}
    p \triangleq \mathbb{P}\left(\hat{\mu}_{\texttt{I},i} \ge \bar{f}+\frac{d \bar{\sigma}}{\sqrt{\frac{N}{\tau}}}\right) \le \frac{4L}{d^2} \le \frac{1}{4}.
    \end{equation}
    
\textbf{Step 4. Bounding the performance of \texttt{RUMEM} under i.i.d. data}: Similar to Step 3, consider again the scenario when the sub-sampled sequence is generated in an i.i.d. manner. Now with each non-contaminated bucket $\mc{B}_i$, let us associate an indicator random variable $Y_i$, such that $Y_i =1$ if $\hat{\mu}_{\texttt{I},i} \ge \bar{f} + \frac{d \bar{\sigma}}{\sqrt{\frac{N}{\tau}}}$, and $0$ otherwise. From Step 3, we know that $\mathbb{E}[Y_i]=p \leq 1/4.$ To proceed, we will find it useful to define a couple of events. By $\mc{C}$, we define an event where $\exists$ $\frac{L}{2}$ buckets $\mathcal{B}_i, i \in \{1,L\}$, such that the corresponding means of those buckets satisfy $\hat{\mu}_{\texttt{I},i}\ge \bar{f}+\frac{d \bar{\sigma}}{\sqrt{\frac{N}{\tau}}}$. We also define an event $\mc{D}$ where $ \exists \left(\frac{L}{2}-\frac{3\epsilon^\prime n}{2}\right)$ non-contaminated buckets, such that each such bucket satisfies the same property as above. Noting that on the event $\mc{E}_2$, at most $\frac{3\epsilon^\prime n}{2}$ buckets can be corrupted, we then have:
 \begin{equation}
 \label{eqn:final_iid}
 \begin{aligned}
    \mathbb{P}\left(\hat{\mu}_{\texttt{I}} \ge \bar{f}+\frac{d \bar{\sigma}}{\sqrt{\frac{N}{\tau}}}\right) & \le \mathbb{P}(\mathcal{C})\\
            & \leq  \mathbb{P}\left(\{\mathcal{C} \cap \mathcal{E}_2\}\right) + \mathbb{P}\left(\mathcal{E}_2^{c}\right)\\
            & \leq \mathbb{P}(\mathcal{D}) + \mathbb{P}\left(\mathcal{E}_2^{c}\right).
\end{aligned}
    \end{equation}
Our next goal is to establish an upper bound on $\mathbb{P}(\mathcal{D})$. To that end, let $\mathcal{J}$ denote the set of indices corresponding to the non-contaminated buckets, and let $\tilde{N}= |\mc{J}|.$ We then have:
\begin{equation}\label{eqn:nohoeffdingdirectly}
\begin{aligned}
        & \mathbb{P}(\mathcal{D}) = \mathbb{P} \left( \frac{1}{\tilde{N}}\sum_{i \in \mathcal{J}} Y_i \ge \frac{\frac{L}{2}-\frac{3\epsilon^\prime n}{2}}{\tilde{N}} \right)\\
        & \overset{(a)}\le \mathbb{P} \left( \frac{1}{\tilde{N}}\sum_{i \in \mathcal{J}} Y_i \ge \frac{\frac{L}{2}-\frac{3\epsilon^\prime n}{2}}{L} \right)\\
        & \le \mathbb{P} \left( \frac{1}{\tilde{N}}\sum_{i \in \mathcal{J}} Y_i - p \ge \frac{1}{2}-\frac{3\epsilon^\prime n}{2L}-p\right)\\
        & \overset{(b)}\le \mathbb{P} \left( \frac{1}{\tilde{N}}\sum_{i \in \mathcal{J}} Y_i - p \ge \frac{1}{8}\right)\\
        &= \mathbb{P}(\mathcal{F}),
\end{aligned}
\end{equation}
where 
$$ \mc{F} = \Bigl\{\frac{1}{\tilde{N}}\sum_{i \in \mathcal{J}} Y_i - p \ge \frac{1}{8} \Bigr\}.$$
In the above steps, for $(a)$, we used $\tilde{N} \le L$, and for $(b)$, we used  $p \le \frac{1}{4}$ and $L \ge 12 \epsilon^\prime n$. At this stage, one might be tempted to use a Hoeffding bound to control $\mathbb{P}(\mc{F})$. However, care needs to be taken here since $\tilde{N}$ is random. As such, some more work is needed before one can apply Hoeffding's inequality. We proceed by using the law of total probability to obtain:
\begin{equation}\label{eqn:cannotapplyhoeffding}
    \mathbb{P}(\mathcal{F}) = \mathbb{P}(\mathcal{F} \cap \mathcal{E}_2) + \mathbb{P}(\mathcal{F} \cap \mathcal{E}_2^c) \le \mathbb{P}(\mathcal{F} \cap \mathcal{E}_2) + \mathbb{P}(\mathcal{E}_2^c).
\end{equation}
Considering the definition of the event $\mathcal{E}_2$, and using $L \ge 12 \epsilon^\prime n$, we have the following bound on $\tilde{N}$ on the event $\mc{E}_2$: $\tilde{N} \geq L - \frac{3 \epsilon^\prime n}{2} \geq \frac{7L}{8}$. This implies:
\begin{equation}
    \mathbb{P}(\mathcal{F} \cap \mathcal{E}_2) = \sum_{j=\frac{7L}{8}}^{L} \mathbb{P}(\mathcal{F} \cap \mathcal{E}_2 \cap \{\tilde{N} = j\}).
\end{equation}
Combining the above bound with those in equations~\eqref{eqn:nohoeffdingdirectly} and~\eqref{eqn:cannotapplyhoeffding}, we obtain:
\begin{equation}\label{eqn:canapplyhoeffding}
    \begin{aligned}
        \mathbb{P}(\mathcal{D}) &\leq  \sum_{j=\frac{7L}{8}}^{L} \mathbb{P}(\mathcal{F} \cap \mathcal{E}_2 \cap \{\tilde{N} = j\}) + \mathbb{P}(\mathcal{E}_2^c)\\
        & \le \sum_{j=\frac{7L}{8}}^{L} \mathbb{P}(\mathcal{F} \cap \{\tilde{N} = j\}) + \mathbb{P}(\mathcal{E}_2^c)\\
        & \le \sum_{j=\frac{7L}{8}}^{L} \mathbb{P}\left(\frac{1}{j} \sum_{i \in \mathcal{J}, \lvert\mathcal{J}\rvert = j}Y_i - p \ge \frac{1}{8}\right) + \mathbb{P}(\mathcal{E}_2^c). \\
    \end{aligned}
\end{equation}
We can now use Hoeffding's inequality (Lemma~\ref{lemma:hoeffding}) to bound the R.H.S. of Eq.~\eqref{eqn:canapplyhoeffding} since $j$ is deterministic. This yields: 
\begin{equation}
\begin{aligned}
    &\mathbb{P}(\mathcal{D}) \le \sum_{j=\frac{7L}{8}}^{L} e^{-\frac{j}{32}}+ \mathbb{P}(\mathcal{E}_2^c).\\
    & \hspace{7.5 mm}\le \frac{L}{8}e^{-\frac{7L}{256}}+ \mathbb{P}(\mathcal{E}_2^c).
\end{aligned}
\end{equation}
Finally, we combine the above bound with Eq.~\eqref{eqn:bern_mom} and Eq.~\eqref{eqn:final_iid} to obtain
\begin{equation} \mathbb{P}\left(\hat{\mu}_{\texttt{I}}\ge \bar{f}+\frac{d \bar{\sigma}}{\sqrt{\frac{N}{\tau}}}\right) \le \frac{L}{8}e^{-\frac{7L}{256}}+ 2\mathbb{P}(\mathcal{E}_2^c) = \frac{L}{8}e^{-\frac{7L}{256}} + 4 e^{-\frac{3\epsilon^\prime n}{32}}.
\label{eqn:RUMEMIID}
\end{equation}

This completes the analysis of \texttt{RUMEM} on i.i.d. data. 

\textbf{Step 5. Bounding the performance of \texttt{RUMEM} under Markov data}: In this last step, we will extend our bound for \texttt{RUMEM} with i.i.d. data to the Markov setting by appealing to Lemma~\ref{lemm:coupled}. To see how this can be done, recall that $\hat{\mu}$ is the final \texttt{MoM} estimate under Markov data, and observe
\begin{equation}\label{eqn:finalrumem}\medmath{
\begin{aligned}
    \mathbb{P}\left(\hat{\mu} \ge \bar{f}+\frac{d \bar{\sigma}}{\sqrt{\frac{N}{\tau}}}\right) &\le \mathbb{P}\left(\Bigl\{\hat{\mu} \ge \bar{f}+\frac{d \bar{\sigma}}{\sqrt{\frac{N}{\tau}}}\Bigr\} \cap \mathcal{E}_1 \right) + \mathbb{P}\left( \mathcal{E}_1^c \right)\\
    & \leq  \mathbb{P}\left(\hat{\mu}_{\texttt{I}} \ge \bar{f}+\frac{d \bar{\sigma}}{\sqrt{\frac{N}{\tau}}}\right) + \mathbb{P}\left( \mathcal{E}_1^c \right)\\
    & \overset{(a)}\le \frac{L}{8}e^{-\frac{7L}{256}}+ 4 e^{-\frac{3\epsilon^\prime n}{32}} + \mathbb{P}\left( \mathcal{E}_1^c \right)\\
    & \overset{(b)}\le \frac{L}{8}e^{-\frac{7L}{256}}+ 4 e^{-\frac{3\epsilon^\prime n}{32}} + (n-1) d_{mix}(\tau).
\end{aligned}}
\end{equation}

For (a), we used Eq.~\eqref{eqn:RUMEMIID}, and for (b), we invoked Lemma~\ref{lemm:coupled}. Our goal is to now ensure that each term appearing on the R.H.S. of the above display is bounded from above by $\delta/6$. Let us start with the term $(n-1) d_{mix}(\tau)$. From~\cite{dorfman}, we know that for any positive integer $\ell \in \mathbb{N}$, if $\tau = \ell \tau_{mix}$, then $d_{mix}(\tau) \leq 2^{-\ell}$. Thus, picking $\tau = \ceil{\log_2\left(6N/\delta\right)\tau_{mix}}$, we obtain
$$ (n-1) d_{mix}(\tau) \leq N \cdot 2^{- \log_2\left(6N/\delta\right)} \leq \frac{\delta}{6}.$$

Next, given our choice of $\epsilon^\prime = \epsilon + \frac{32}{3n}\log(\frac{24}{\delta})$, straightforward calculations reveal that 
$$ 4 e^{-\frac{3\epsilon^\prime n}{32}} \leq \frac{\delta}{6}. $$

Finally, given that $\frac{L}{8}e^{-\frac{7L}{256}} \le \frac{N}{8}e^{-\frac{7L}{256}}$, it is easy to verify that by picking $L$ to satisfy
$$ L \geq \frac{256}{7} \log\left(\frac{N}{\delta}\right),$$ 
one can ensure that the first term in the R.H.S. of Eq.~\eqref{eqn:finalrumem} is at most $\delta/6$. Combining the prior requirement $L \geq 12 \epsilon' n$ on $L$ from step 4 with the one above, it suffices to set $L=\ceil{12 \epsilon^\prime n + \frac{256}{7} \log(\frac{N}{\delta})}.$ We conclude that the R.H.S of Eq.~\eqref{eqn:finalrumem} is at most $\delta/2$. Using a symmetric argument for the lower tail, we have that with probability at least $1-\delta$, the following is true:
\begin{equation}\label{eqn:almostfinal}
    |\hat{\mu}-\bar{f}| \le d \bar{\sigma}  \sqrt{\frac{\tau}{N}}.
\end{equation}
Recalling $d = 4\sqrt{L}$ from Step 2, and using the expression for $L$, we further have that with probability at least $1-\delta$, 
\begin{equation}\label{eqn:finalboundrumem}
    \begin{aligned}
        \rvert \hat{\mu}-\bar{f} \lvert &\le 
 O\left(  \bar{\sigma}  \sqrt{\frac{\tau}{N}}\right)\sqrt{12 \epsilon^\prime n + \frac{256}{7} \log\left(\frac{N}{\delta}\right)}\\
        & \overset{(a)}\le O\left(\bar{\sigma}\sqrt{\frac{\tau}{N}}\right)\left(\sqrt{\epsilon^\prime n} + \sqrt{\log\left(\frac{N}{\delta}\right)}\right)\\
        &\overset{(b)}\le O\left(\bar{\sigma}\sqrt{\frac{\tau}{N}}\right)\left(\sqrt{\epsilon n} + \sqrt{ \log\left(\frac{24}{\delta}\right)} +\sqrt{\log\left(\frac{N}{\delta}\right)}\right)\\
        &\overset{(c)} \le O(\bar{\sigma})\left(\sqrt{\epsilon} +\sqrt{\frac{\tau}{N}\log\left(\frac{N}{\delta}\right)}\right),
    \end{aligned}
\end{equation}
where for (a) and (b), we used the fact that for positive, real scalars $\alpha,\beta$, the following is true:  $\sqrt{\alpha+\beta} \le \sqrt{\alpha} + \sqrt{\beta}$. Finally, for $(c)$, we used $\sqrt{n} \le 2\sqrt{\frac{N}{\tau}}$. This completes the proof. 
\section{Main Convergence Analysis for \texttt{Robust-TD}: Proof of Theorem~\ref{thm:main_result}}
\label{app:Mainresultproof}
In this section, we will prove our main convergence result for \texttt{Robust-TD}, namely Theorem~\ref{thm:main_result}. The key new technical ingredient that we need in our analysis is a robust estimate of the object $\bar{b}$ that features in the mean-path TD update direction $\bar{g}(\theta)=\bar{A} \theta + \bar{b}$. This is achieved in the following lemma.

\begin{lemma} (\textbf{Adversarial Perturbation Bound}) 
\label{lemm:rob_est} Suppose Assumption~\ref{ass:MC} holds, and the initial distribution of $s_0$ is the steady-state distribution $\pi.$ There exists a universal constant $c >0$, such that if the burn-in time $\bar{T}$ satisfies the requirement in Theorem~\ref{thm:main_result}, 
then the following is true for all $t \geq \bar{T}$:
\begin{equation}\begin{aligned}
    &\mathbb{E}[ \Vert \hat{b}_t - \bar{b} \Vert^2] \leq c\left(\epsilon + \frac{\log^2(KT) \tau_{mix}}{t} \right) K \sigma^2_1, \\
    & \hspace{2mm} \textrm{where} \hspace{2mm} \sigma_1 = \max\{1, \bar{r}, \rho\}.
\end{aligned}
\label{eqn:hatb_finalbnd}
\end{equation}
\end{lemma}

\begin{proof}
The proof comprises three steps. In the first step, we use the guarantees from the \texttt{RUMEM} estimator in Theorem~\ref{theorem:RobustMean} to establish a high-probability bound on the error $\Vert \hat{b}_t - \bar{b} \Vert$. In the second step, we use the result from step 1 to argue that with high probability, the resetting operating in line 8 of Algorithm~\ref{algo:algoRobustTD} gets bypassed, and $\hat{b}_t$ corresponds to the output of the robust estimation procedure in line 6 of Algorithm~\ref{algo:algoRobustTD}. Finally, in the last step, we establish a mean-square error bound by leveraging the thresholding operation in line 8 of \texttt{Robust-TD}. We now proceed to provide the details for each of these steps.

\textbf{Step 1: A high-probability estimate on $\Vert \hat{b}_t - \bar{b} \Vert$.} Let us start by fixing a component $i \in [K]$ of $\bar{b}$ and $\hat{b}_t$, and a time-step $t \geq \bar{T}$. Now consider the estimation process in line 6 of Algorithm~\ref{algo:algoRobustTD}: $[\hat{b}_t]_i = \textrm{\texttt{RUMEM}}\left( \{y_{i,k}\}_{0 \leq k \leq t}; \delta = 1/(KT^2)\right)$, and $y_{i,k} = [\phi(s_k)]_i \tilde{r}_k$. We wish to relate this estimation step to the robust mean estimation set up in Section~\ref{sec:RobustMeanEst}. To that end, let us make the following observations by considering the data set $\mc{S} =  \{y_{i,k}\}_{0 \leq k \leq t}$. First, note that since the initial distribution of $s_0$ is the stationary distribution $\pi$, the resulting Markov chain $\{s_0, s_1, \ldots\}$ induced by the policy $\mu$ is stationary. Thus, $s_t \sim \pi, \forall t.$ Next, let us consider the statistics of a sample $y_{i,k}$ that is not corrupted, i.e., a sample for which $\tilde{r}_k = r(s_k) \sim \mc{D}_{\mu}(\cdot| s_k)$. For such a sample, we have:
\begin{equation}
\begin{aligned}
&\mathbb{E}_{s_k \sim \pi, r(s_k) \sim \mc{D}_{\mu}(\cdot| s_k)}[ y_{i,k} ]\\ 
&= \sum_{s \in \mc{S}} \mathbb{E}_{r(s_k) \sim \mc{D}_{\mu}(\cdot| s_k)} [ [\phi(s_k)]_i r(s_k) | s_k = s ] \pi(s) \\
&= \sum_{s \in \mc{S}} [\phi(s)]_i \mathbb{E}_{{r(s) \sim \mc{D}_{\mu}(\cdot| s)}}[ r(s) ] \pi(s)\\
&= \sum_{s \in \mc{S}} [\phi(s)]_i R_{\mu}(s) \pi(s)\\
&= [\bar{b}]_i.
\end{aligned}
\end{equation}
In other words, the mean of an uncorrupted sample corresponds exactly to the $i$-th component of $\bar{b}$. Proceeding as above, we have:
\begin{equation}
\begin{aligned}
&\mathbb{E}_{s_k \sim \pi, r(s_k) \sim \mc{D}_{\mu}(\cdot| s_k)}[ y^2_{i,k} ]\\
&= \sum_{s \in \mc{S}} \mathbb{E}_{r(s_k) \sim \mc{D}_{\mu}(\cdot| s_k)} [ ([\phi(s_k)]_i r(s_k))^2 | s_k = s ] \pi(s) \\
&= \sum_{s \in \mc{S}} ([\phi(s)]_i)^2  \mathbb{E}_{{r(s) \sim \mc{D}_{\mu}(\cdot| s)}}[ r^2(s) ] \pi(s)\\
& \overset{(a)}\leq \sum_{s \in \mc{S}}  \mathbb{E}_{{r(s) \sim \mc{D}_{\mu}(\cdot| s)}}[ r^2(s) ] \pi(s)\\
& \overset{(b)}\leq  \sum_{s \in \mc{S}} ({R}^2_{\mu}(s) + \rho^2) \pi(s)\\
& \overset{(c)}\leq  \sum_{s \in \mc{S}} (\bar{r}^2 + \rho^2) \pi(s)\\
& \overset{(d)}\leq 2 \sigma^2_1. 
\end{aligned}
\end{equation}
In the above steps, for (a), we used the fact that $\Vert \phi(s) \Vert^2 \leq 1, \forall s \in \mc{S}.$ For (b), we used that the variance of the random variable $r(s)$ is upper-bounded by $\rho^2$. To arrive at (c), we used the uniform upper bound on the means of the rewards: $|R_{\mu}(s)| \leq \bar{r}, \forall s \in \mc{S}.$ Finally, for (d), we used $\sigma_1 = \max\{1, \bar{r}, \rho\}.$ We conclude that for each sample in the data set $\mc{S}$, with probability $1-\epsilon,$ we observe a ``clean" random variable with mean $[\bar{b}]_i$, and variance at most $2 \sigma^2_1.$ 

Given that the \texttt{RUMEM} sub-routine is invoked in line 6 of Algorithm~\ref{algo:algoRobustTD} with $\delta = 1/(KT^2)$, and number of samples $N=(t+1)$, we have from Theorem~\ref{theorem:RobustMean} that with probability at least $1-1/(KT^2)$,
\begin{equation}
\begin{aligned}
&\left\lvert [\hat{b}_t]_i - [\bar{b}]_i\right\rvert  \leq  C \sigma_1 \left( \sqrt{\epsilon} + \sqrt{\frac{\tau}{N}\log\left(\frac{N}{\delta}\right)} \right) \\
&\leq  C \sigma_1 \left( \sqrt{\epsilon} + \sqrt{\frac{2 \tau_{mix} \log_2(6NKT^2)\log\left(N K T^2\right)}{t}} \right)\\
& \leq  C \sigma_1 \left( \sqrt{\epsilon} + 2 \log(12KT^3)\sqrt{\frac{\tau_{mix}}{t}} \right),
\end{aligned}
\label{eqn:rob_est_bnd1}
\end{equation}
where we used the expression for $\tau$ in Eq.~\eqref{eqn:paramchoices}, and the fact that $N=t+1 \leq 2T$. Union-bounding over each component $i \in [K]$, and over all time-steps $t \geq \bar{T}$, we have that with probability at least $1 - 1/T$, 
\begin{equation}\begin{aligned}
\left\lvert [\hat{b}_t]_i - [\bar{b}]_i\right\rvert &\leq C \sigma_1 \left( \sqrt{\epsilon} + 2 \log(12KT^3)\sqrt{\frac{\tau_{mix}}{t}} \right),\\
&\forall i \in [K], \forall t \geq \bar{T}.\end{aligned}
\end{equation}
Let us call the event on which the above inequalities hold $\mc{E}$. It then follows that on the event $\mc{E}$, the following is true:
\begin{equation}
\Vert \hat{b}_t - \bar{b} \Vert \leq C \sqrt{K} \sigma_1 \left( \sqrt{\epsilon} + 2 \log(12KT^3)\sqrt{\frac{\tau_{mix}}{t}} \right), \forall t \geq \bar{T}.
\label{eqn:high_prob_bnd}
\end{equation}

To complete step 1, we note that for us to be able to invoke Theorem~\ref{theorem:RobustMean} and arrive at the bound in Eq.~\eqref{eqn:rob_est_bnd1}, we need the number of samples in $\mc{S}$, namely $N = t+1$, to satisfy the requirement $N \geq 4 L \tau$. Here,  recall from the description of \texttt{RUMEM} that $L$ is the number of buckets, and $\tau$ is the sub-sampling gap. Using the expressions for $\tau$ and $L$ in Eq.~\eqref{eqn:paramchoices}, along with $N = t+1 \leq 2T$ and $ \delta = 1/(KT^2)$, we have that 
$$ 4 L \tau \leq 96 \epsilon N + c' \tau_{mix} \log^2 (KT),$$
where $c'$ is some suitably large universal constant. So the requirement that $N \geq 4 L \tau$ is met if $\epsilon \leq 1/(192)$ and $t \geq 2 c' \tau_{mix} \log^2 (KT)$ - the latter requirement is taken care of by the choice of the burn-in time $\bar{T}$ in Theorem~\ref{thm:main_result}. This concludes step 1. 

\textbf{Step 2.} Next, we claim that on the good event $\mc{E}$, line 8 of Algorithm~\ref{algo:algoRobustTD} will always get bypassed, and $\hat{b}_t$ will be the output of the estimation scheme in line 6 of Algorithm~\ref{algo:algoRobustTD}. To see this, we start by noting that 
$$ \bar{b} = \sum_{s\in \mc{S}} \phi(s) R_{\mu}(s) \pi(s).$$
Thus,
\begin{equation}
\begin{aligned}
\Vert \bar{b} \Vert &\leq \sum_{s \in \mc{S}} \Vert \phi(s) \Vert |R_{\mu}(s)| \pi(s)\\
&\leq \sum_{s\in \mc{S}} \bar{r} \pi(s)\\
&\leq \sigma_1,
\end{aligned}
\label{eqn:bnd_barb}
\end{equation}
where for second inequality, we used $\Vert \phi(s) \Vert \leq 1, \forall s\in \mc{S}$, and for the third inequality, we used $|R_{\mu}(s)| \leq \bar{r} \leq \sigma_1.$ Combining the above observation with Eq.~\eqref{eqn:high_prob_bnd}, we conclude that on the event $\mc{E}$, the following is true:
$$\Vert \hat{b}_t \Vert  \leq \underbrace{C \sqrt{K} \sigma_1 \left( \sqrt{\epsilon} + 2 \log(12KT^3)\sqrt{\frac{\tau_{mix}}{t}} \right)}_{G_t} + \sigma_1, \forall t \geq \bar{T}.$$
This immediately leads to the claim that line 8 of Algorithm~\ref{algo:algoRobustTD} always gets bypassed on event $\mc{E}.$

\textbf{Step 3. Bound on the expected value of $\Vert \hat{b}_t - \bar{b} \Vert^2$.} Let us start by noting that if $\Vert \hat{b}_t \Vert > G_t + \sigma_1$, then as per the thresholding operation in line 8 of Algorithm~\ref{algo:algoRobustTD}, $\hat{b}_t$ gets reset to $0$. In this case, $\Vert \hat{b}_t - \bar{b} \Vert = \Vert \bar{b} \Vert \leq \sigma_1$, where in the last step, we used Eq.~\eqref{eqn:bnd_barb}. We conclude that thanks to the thresholding operation, the following is always true deterministically: 
\begin{equation}
\Vert \hat{b}_t - \bar{b} \Vert \leq G_t + 2\sigma_1, \forall t \geq \bar{T}. 
\end{equation}
Furthermore, from the requirement on $\bar{T}$ in Theorem~\ref{thm:main_result}, we have that $\bar{T} \geq \tau_{mix} (2 \log(12KT^3)^2.$ This tells us that for $t \geq \bar{T}, G_t \leq 2 \sqrt{K} C \sigma_1.$ We then have that
\begin{equation}
\Vert \hat{b}_t - \bar{b} \Vert \leq G_t + 2\sigma_1 \leq 2(\sqrt{K}C +1) \sigma_1, \forall t \geq \bar{T}.
\label{eqn:bnd_hatb}
\end{equation}
We are now in a position to bound $\mathbb{E}[ \Vert \hat{b}_t - \bar{b} \Vert^2 ]$. Using $\mathbf{1}_{\mc{V}}$ as an indicator for any event $\mc{V}$, we have $\forall t \geq \bar{T}$: 
\begin{equation}\medmath{
\begin{aligned}
\mathbb{E}[ \Vert \hat{b}_t - \bar{b} \Vert^2 ] &= \mathbb{E}[ \Vert \hat{b}_t - \bar{b} \Vert^2 \mathbf{1}_{\mc{E}}]  + \mathbb{E}[ \Vert \hat{b}_t - \bar{b} \Vert^2 \mathbf{1}_{\mc{E}^c}]\\
& \overset{(a)}\leq \mathbb{E}[ \Vert \hat{b}_t - \bar{b} \Vert^2 \mathbf{1}_{\mc{E}}] + \left(2(\sqrt{K}C +1)\sigma_1\right)^2 \mathbb{E}[\mathbf{1}_{\mc{E}^c}]\\
& = \mathbb{E}[ \Vert \hat{b}_t - \bar{b} \Vert^2 \mathbf{1}_{\mc{E}}] + \left(2(\sqrt{K}C +1)\sigma_1\right)^2 \mathbb{P}(\mc{E}^c)\\
& \overset{(b)} \leq \mathbb{E}[ \Vert \hat{b}_t - \bar{b} \Vert^2 \mathbf{1}_{\mc{E}}] + \frac{(\sqrt{K}C +1)^2 4\sigma^2_1}{T}\\
& \overset{(c)} \leq O(\sigma_1^2 K \epsilon) + O\left( \frac{\log^2(12 K T^3) \sigma^2_1 K \tau_{mix}}{t}\right) + \frac{\sigma^2_1 K}{T}
\\ 
& \leq O\left(\left(\epsilon + \frac{\log^2(KT) \tau_{mix}}{t}\right) K \sigma^2_1 \right). 
\end{aligned}}
\end{equation}
In the above steps, for (a), we used Eq.~\eqref{eqn:bnd_hatb} to bound $\Vert \hat{b}_t - \bar{b} \Vert$ on the event $\mc{E}^c$. For (b), we used $\mathbb{P}(\mc{E}^c) \leq 1/T$. Finally, in view of steps 1 and 2, we used Eq.~\eqref{eqn:high_prob_bnd} to bound $\Vert \hat{b}_t - \bar{b} \Vert$ on the event $\mc{E}.$ This concludes the proof. 
\end{proof}

Equipped with the above lemma, our next step is establishing a one-step mean-square error decomposition. Before we do so, we remind the reader here of some notation: recall that $d_t = \Vert \theta_t - \theta^* \Vert^2$,  $\tilde{g}_t(\theta) = A_t \theta + \hat{b}_t,$ and $\sigma = \max\{\sigma_1, \Vert \theta^* \Vert, \Vert \theta_0 \Vert\}.$ Given the guarantee from Lemma~\ref{lemm:rob_est}, we will also find it useful to employ the following notation: 
\begin{equation}
 B_t = c\left(\epsilon + \frac{\log^2(KT) \tau_{mix}}{t} \right) K \sigma^2_1.
 \label{eqn:Bt}
\end{equation}
We have the following result.

\begin{lemma} (\textbf{Main Recursion}) 
\label{lemma:recursion}
Suppose the conditions in the statement of Theorem~\ref{thm:main_result} are met. Then, the following is true $\forall t \geq \bar{T}$:
\begin{equation}
\mathbb{E}[d_{t+1}] \leq (1 - \alpha \beta + 12\alpha^2) \mathbb{E}[d_t] + O(\alpha^2 \sigma^2 K) + \frac{\alpha B_t}{\beta} + \mathbb{E}[M_t],
\end{equation}
where $\beta = \omega (1 - \gamma)$, $M_t = 2 \alpha \langle \theta_t - \theta^*, (A_t - \bar{A}) \theta_t \rangle,$ and $B_t$ is as in Eq.~\eqref{eqn:Bt}.
\end{lemma}
\begin{proof}
From the update rule of \texttt{Robust-TD} in line 10 of Algorithm~\ref{algo:algoRobustTD}, we have:
\begin{equation}\medmath{
\begin{aligned}
\Vert \theta_{t+1} - \theta^* \Vert^2 &= \Vert \theta_t - \theta^* \Vert^2 + 2\alpha \langle \theta_t - \theta^*, \tilde{g}_t (\theta_t) \rangle + \alpha^2 \Vert \tilde{g}_t (\theta_t) \Vert^2\\
&= \Vert \theta_t - \theta^* \Vert^2 + 2\alpha \langle \theta_t - \theta^*, \bar{g} (\theta_t) \rangle \\
& + \alpha^2 \Vert \tilde{g}_t (\theta_t) \Vert^2 + 2\alpha \langle \theta_t - \theta^*, \tilde{g}_t (\theta_t) - \bar{g}(\theta_t) \rangle\\ 
&= \underbrace{\Vert \theta_t - \theta^* \Vert^2 + 2\alpha \langle \theta_t - \theta^*, \bar{g} (\theta_t) \rangle}_{(*)} + \underbrace{\alpha^2 \Vert \tilde{g}_t (\theta_t) \Vert^2}_{(**)}\\
& \hspace{2mm} + \underbrace{2\alpha \langle \theta_t - \theta^*, (A_t - \bar{A}) \theta_t \rangle}_{(***)} + \underbrace{2\alpha \langle \theta_t - \theta^*, \hat{b}_t - \bar{b} \rangle}_{(****)}. 
\end{aligned}}
\label{eqn:errdc_0}
\end{equation}
Before proceeding further, it is instructive to take a moment to interpret each of the terms in the above decomposition. The term $(*)$ captures the ``steady-state" behavior of \texttt{TD}(0) and is responsible for driving the iterates toward $\theta^*$. This term can be analyzed using Lemma~\ref{lemma:convex}, yielding:
\begin{equation}
(*) \leq (1 - 2 \alpha \beta) d_t. 
\label{eqn:errdc_1}
\end{equation}
Each of the remaining three terms can be viewed as a perturbation/disturbance to the nominal/steady-state dynamics. As for $(**)$ and $(****)$, since they feature $\hat{b}_t$, which is processed based on contaminated rewards, these terms capture the effects of adversarial perturbations. On the other hand, the term $(***)$ depends on the ``closeness" between $A_t$ and its steady-state version $\bar{A}$, which, in turn, is dictated by how quickly the underlying Markov chain is mixing. In the remainder of the proof, we bound the terms that contain adversarial effects. We start by noting that thanks to the thresholding operation, we have that $\Vert \hat{b}_t \Vert \leq G_t + \sigma_1 \leq G_t + \sigma, \forall t \geq \bar{T}.$ Furthermore, in the analysis of Lemma~\ref{lemm:rob_est}, we argued that for all $t\geq \bar{T}$, it holds that $G_t \leq 2 \sqrt{K} C \sigma_1 \leq 2 \sqrt{K} C \sigma,$ where $C$ is some universal constant. Thus, we have 
\begin{equation}
\Vert \hat{b}_t \Vert \leq (2\sqrt{K} C +1) \sigma = O(\sqrt{K} \sigma), \forall t \geq \bar{T}.
\label{eqn:hatb_final}
\end{equation}
To see how the above bound helps us, observe:
\begin{equation}
\begin{aligned}
(**) &= \alpha^2 \Vert A_t \theta_t + \hat{b}_t \Vert^2\\
&= \alpha^2 \Vert A_t (\theta_t-\theta^*) + A_t \theta^* + \hat{b}_t \Vert^2\\
&\leq 3 \alpha^2 \Vert A_t \Vert^2 d_t + 3 \alpha^2 \Vert A_t \Vert^2 \Vert \theta^* \Vert^2 + 3 \alpha^2 \Vert \hat{b}_t \Vert^2\\
&\leq  12 \alpha^2 d_t + 12 \alpha^2 \sigma^2 + 3 \alpha^2 \Vert \hat{b}_t \Vert^2\\
& \leq 12 \alpha^2 d_t + O(\alpha^2 K \sigma^2),
\end{aligned}
\label{eqn:errdc_2}
\end{equation}
where we used $\Vert A_t \Vert \leq 2$ and $\Vert \theta^* \Vert \leq \sigma.$ Now for the term $(****)$, we have
\begin{equation}
\begin{aligned}
(****) &\leq \alpha \beta\Vert \theta_t - \theta^* \Vert^2 + \frac{\alpha}{\beta} \Vert \hat{b}_t - \bar{b} \Vert^2\\
& = \alpha \beta d_t + \frac{\alpha}{\beta} \Vert \hat{b}_t - \bar{b} \Vert^2. 
\end{aligned}
\end{equation}
Taking expectations on both sides of the above inequality, and using Lemma~\ref{lemm:rob_est}, we have
$$ \mathbb{E}[(****)] \leq \alpha \beta \mathbb{E}[d_t] + \frac{\alpha 
 B_t}{\beta}. $$
Taking expectations on both sides of Eq.~\eqref{eqn:errdc_0}, and then combining the bounds on $(*), (**),$ and $(****)$ from Eq.~\eqref{eqn:errdc_1},~\eqref{eqn:errdc_2}, and the above display, respectively, leads to the claim of the lemma. 
\end{proof}

From Lemma~\ref{lemma:recursion}, it is clear that in order to proceed further, we need to bound the term $\mathbb{E}[M_t]$ that corresponds to the bias introduced by Markov noise. To that end, we will require an intermediate result. Before stating this result, we remind the reader that $\tau'= \tau'_{mix}(\alpha)$ is as defined in Section~\ref{sec:results}. 

\begin{lemma} (\textbf{Bounding the Drift}) 
\label{lemm:drift}
Suppose the conditions in the statement of Theorem~\ref{thm:main_result} are met. Then, the following bound holds $\forall t \geq \bar{T}+\tau'$: 
\begin{equation}
\Vert \theta_t - \theta_{t-\tau'} \Vert^2 \leq O(\alpha^2 \tau'^2) d_t + O(\alpha^2 \tau'^2 K \sigma^2 ). 
\end{equation}
\end{lemma}
\begin{proof}
From Eq.~\eqref{eqn:hatb_final}, recall that $\Vert \hat{b}_t \Vert \leq O(\sqrt{K} \sigma), \forall t \geq \bar{T}.$ Using this, we obtain
\begin{equation}
\begin{aligned}
\Vert \theta_{t+1} \Vert &\leq \Vert \theta_t \Vert + \alpha \Vert \tilde{g}_t (\theta_t) \Vert\\  
& \leq \Vert \theta_t \Vert + \alpha \left(\Vert A_t\Vert \Vert \theta_t \Vert + \Vert \hat{b}_t \Vert \right)\\ 
& \leq (1+2\alpha) \Vert \theta_t \Vert + O(\alpha \sqrt{K} \sigma).
\end{aligned}
\end{equation}
Rolling out the above recursion, we obtain the following for any $k \in [t-\tau', t]:$
$$ \Vert \theta_k \Vert \leq (1+2 \alpha)^{\tau'} \Vert \theta_{t-\tau'} \Vert  + O(\alpha \sqrt{K} \sigma) \sum_{\ell =0}^{\tau'}(1+2\alpha)^{\ell}.$$
Since $(1+x) \leq \exp(x), \forall x \in \mathbb{R}$, note that $(1+2\alpha)^{\tau'} \leq \exp(0.25) < 2$, for $\alpha \leq 1/(8\tau').$ Using this to simplify the above display, we have that for any $k \in [t-\tau', t]$, the following is true:
\begin{equation}
\Vert \theta_k \Vert \leq 2 \Vert \theta_{t-\tau'} \Vert + O(\alpha \tau' \sqrt{K} \sigma) \leq 2 \Vert \theta_{t-\tau'} \Vert + O(\sqrt{K} \sigma),
\label{eqn:drift_int}
\end{equation}
where in the last step, we used $\alpha \tau' \leq 1.$
Let us now observe the following chain of inequalities:
\begin{equation}
\begin{aligned}
\Vert \theta_{t} - \theta_{t-\tau'} \Vert &\leq \sum_{k= t-\tau'}^{t-1} \Vert \theta_{k+1} - \theta_k \Vert\\
& \leq \alpha \sum_{k= t-\tau'}^{t-1} \Vert \tilde{g}_k (\theta_k) \Vert\\
& \leq \alpha \sum_{k= t-\tau'}^{t-1} \left(\Vert A_k \Vert \Vert \theta_k \Vert + \Vert \hat{b}_k \Vert \right)\\
&\overset{(a)}\leq \alpha \sum_{k= t-\tau'}^{t-1} \left(2 \Vert \theta_k \Vert + O(\sqrt{K} \sigma) \right)\\
&\overset{(b)} \leq \alpha \sum_{k= t-\tau'}^{t-1}  \left(4 \Vert \theta_{t-\tau'} \Vert + O(\sqrt{K} \sigma)\right)\\
& \leq 4 \alpha \tau' \Vert \theta_{t-\tau'} \Vert + O(\alpha \tau' \sqrt{K} \sigma)\\
& \leq 4 \alpha \tau' \left( \Vert \theta_{t} - \theta_{t-\tau'} \Vert + \Vert \theta_{t} \Vert \right) + O(\alpha \tau' \sqrt{K} \sigma)\\
& \overset{(c)}\leq \frac{1}{2} \Vert \theta_t - \theta_{t-\tau'} \Vert + 4 \alpha \tau' \Vert \theta_{t} \Vert + O(\alpha \tau' \sqrt{K} \sigma).\\ 
\end{aligned}
\label{eqn:drift_int2}
\end{equation}
In the above steps, for (a), we used $\Vert A_k \Vert \leq 2$ and $\Vert \hat{b}_k \Vert \leq O(\sqrt{K} \sigma)$; for (b), we used Eq.~\eqref{eqn:drift_int}; and for (c), we used $\alpha \tau' \leq 1/8.$ Rearranging Eq.~\eqref{eqn:drift_int2} and simplifying, we obtain: 
$$ \medmath{\Vert \theta_t - \theta_{t-\tau'} \Vert \leq 8 \alpha \tau' \Vert \theta_t \Vert + O(\alpha \tau' \sqrt{K} \sigma) \leq  8 \alpha \tau' \Vert \theta_t - \theta^* \Vert + O(\alpha \tau' \sqrt{K} \sigma),}$$
where we used $\Vert \theta^* \Vert \leq \sigma.$ Squaring both sides of the above display leads to the claim of the lemma. 
\end{proof}

With the above lemma in hand, we can now bound the term $\mathbb{E}[M_t]$. 

\begin{lemma} (\textbf{Markovian Bias Bound}) 
\label{lemm:mixing}
Let $M_t$ be as defined in Lemma~\ref{lemma:recursion}. 
Suppose the conditions in the statement of Theorem~\ref{thm:main_result} are met. Then, the following bound holds $\forall t \geq \bar{T}+\tau'$: 
$$\mathbb{E}[M_t] \leq O(\alpha^2 \tau') \mathbb{E}[d_t] + O(\alpha^2 \tau') K \sigma^2. $$  
\end{lemma}
\begin{proof}
Let us start by splitting the term $\langle \theta_t - \theta^*, (A_t - \bar{A}) \theta_t \rangle = T_1 + T_2 + T_3 + T_4$ into the four parts shown below:
\begin{equation}
\begin{aligned}
T_1 &= \langle \theta_t - \theta_{t-\tau'}, (A_t - \bar{A}) \theta_t \rangle\\
T_2 &= \langle \theta_{t-\tau'} - \theta^*, (A_t - \bar{A}) \theta_{t-\tau'} \rangle\\
T_3 &= \langle \theta_{t-\tau'} - \theta^*, A_t (\theta_t - \theta_{t-\tau'}) \rangle\\
T_4 &= \langle \theta_{t-\tau'} - \theta^*, \bar{A} (\theta_{t-\tau'} - \theta_t) \rangle.
\end{aligned}
\end{equation}
In what follows, we proceed to bound each of the four terms above. 

\textbf{Bounding $T_1$.} We bound $T_1$ as follows:
\begin{equation}
    \begin{aligned}
T_1 &\leq \frac{1}{2 \alpha \tau'} \Vert \theta_t - \theta_{t-\tau'} \Vert^2 + \frac{\alpha \tau'}{2} \Vert (A_t - \bar{A}) \theta_t \Vert^2 \\
& \leq \frac{1}{2 \alpha \tau'} \Vert \theta_t - \theta_{t-\tau'} \Vert^2 + \alpha \tau' \left( \Vert A_t \Vert^2 + \Vert \bar{A} \Vert^2 \right) \Vert \theta_t \Vert^2\\
& \overset{(a)} \leq \frac{1}{2 \alpha \tau'} \Vert \theta_t - \theta_{t-\tau'} \Vert^2 + 8 \alpha \tau' \Vert \theta_t \Vert^2\\
& \leq \frac{1}{2 \alpha \tau'} \Vert \theta_t - \theta_{t-\tau'} \Vert^2 + 16 \alpha \tau' \Vert \theta_t - \theta^* \Vert^2 + 16 \alpha \tau' \Vert \theta^* \Vert^2 \\
& \overset{(b)}\leq \frac{1}{2 \alpha \tau'} \Vert \theta_t - \theta_{t-\tau'} \Vert^2 + 16 \alpha \tau' d_t + 16 \alpha \tau' \sigma^2\\
& \overset{(c)} \leq O(\alpha \tau') d_t + O (\alpha \tau') K \sigma^2. 
    \end{aligned}
\end{equation}
Here, for (a), we used $\max\{\Vert A_t \Vert, \Vert \bar{A} \Vert \} \leq 2$; for (b), we used $\Vert \theta^* \Vert \leq \sigma$; and for (c), we used Lemma~\ref{lemm:drift}. 

\textbf{Bounding $T_3$.} To bound $T_3$, we proceed as follows: 
\begin{equation}
\begin{aligned}
T_3 & \leq \Vert \theta_{t-\tau'} - \theta^* \Vert \Vert A_t \Vert \Vert \theta_t - \theta_{t-\tau'} \Vert\\
& \leq 2  \Vert \theta_{t-\tau'} - \theta^* \Vert  \Vert \theta_t - \theta_{t-\tau'} \Vert\\
& \leq \alpha \tau' \Vert \theta_{t-\tau'} - \theta^* \Vert^2 + \frac{1}{\alpha \tau'} \Vert \theta_t - \theta_{t-\tau'} \Vert^2\\
& \leq 2 \alpha \tau' d_t + \left (2 \alpha \tau' + \frac{1}{\alpha \tau'} \right) \Vert \theta_t - \theta_{t-\tau'} \Vert^2\\
& \leq O(\alpha \tau') d_t + O (\alpha \tau') K \sigma^2,
\end{aligned}
\end{equation}
where in the last step, we used Lemma~\ref{lemm:drift} and $\alpha \tau' \leq 1.$ The term $T_4$ can be controlled in exactly the same way as above, with the same resulting bound. 

\textbf{Bounding $T_2$.} To bound $T_2$, we will invoke mixing properties of the underlying Markov chain as follows: 
\begin{equation}
\begin{aligned}
    \mathbb{E}\left[T_{2}\right] &= \mathbb{E}\left[\langle \theta_{t-\tau'} -\theta^*, (A_t - \bar{A}) \theta_{t-\tau'}\rangle\right]\\
    &=\mathbb{E}\left[\mathbb{E}\left[\langle \theta_{t-\tau'} -\theta^*, (A_t - \bar{A}) \theta_{t-\tau'}\rangle | \theta_{t-\tau'}, X_{t-\tau'}\right]\right]\\
    &=\mathbb{E}\left[\langle \theta_{t-\tau'} -\theta^*, \mathbb{E}\left[ (A_t - \bar{A}) \theta_{t-\tau'}| \theta_{t-\tau'}, X_{t-\tau'}\right]\rangle\right]\\
    &= \mathbb{E}\left[\langle \theta_{t-\tau'} -\theta^*, \left(\mathbb{E}\left[A_t| X_{t-\tau'}\right] - \bar{A}\right) \theta_{t-\tau'} \rangle\right]\\
    &\leq \mathbb{E}\left[\Vert \theta_{t-\tau'} -\theta^*\Vert \Vert \left(\mathbb{E}\left[A_t| X_{t-\tau'}\right] - \bar{A}\right) \theta_{t-\tau'}        \Vert\right]\\
    & \leq \mathbb{E}\left[\Vert \theta_{t-\tau'} -\theta^*\Vert \Vert \mathbb{E}\left[A_t| X_{t-\tau'}\right] - \bar{A} \Vert \Vert \theta_{t-\tau'}        \Vert\right]\\
    &\overset{(a)}\leq \alpha \mathbb{E}\left[\Vert \theta_{t-\tau'} -\theta^*\Vert  \Vert \theta_{t-\tau'}\Vert\right]\\
    &\leq \alpha \mathbb{E}\left[\Vert \theta_{t-\tau'} -\theta^*\Vert \left(\Vert \theta^* \Vert + \Vert \theta_{t-\tau'}-\theta^*\Vert\right)\right]\\
    &\leq \alpha \mathbb{E}\left[\Vert \theta_{t-\tau'} -\theta^*\Vert \left(\sigma+\Vert \theta_{t-\tau'} -\theta^* \Vert\right)\right]\\
    & \leq O(\alpha) \mathbb{E}\left[\Vert \theta_{t-\tau'} -\theta^*\Vert^2 + \sigma^2\right]\\
    &\leq O(\alpha) \mathbb{E}\left[\Vert \theta_t - \theta_{t-\tau'}\Vert^2 + d_t +  \sigma^2\right]\\ 
    &\overset{(b)} \leq O(\alpha) d_t + O(\alpha K \sigma^2). 
\end{aligned}
\nonumber
\end{equation}
In the above steps, (a) follows from the definition of the mixing time $\tau' = \tau'_{mix}(\alpha)$ in Definition~\ref{def:mix}, and (b) follows from Lemma~\ref{lemm:drift} and $\alpha \tau' \leq 1.$ Combining the bounds on $T_1 - T_4$ leads to the claim of the lemma. 
\end{proof}
We now have all the ingredients needed to complete the proof of Theorem~\ref{thm:main_result}. 

\begin{proof}
\textbf{Proof of Theorem~\ref{thm:main_result}}. Combining the bound on $\mathbb{E}[M_t]$ from Lemma~\ref{lemm:mixing} with the one-step recursion from Lemma~\ref{lemma:recursion}, we obtain $\forall t \geq \bar{T}+ \tau'$:
\begin{equation}
\begin{aligned}
\mathbb{E}[d_{t+1}] &\leq (1 - \alpha \beta + 12\alpha^2) \mathbb{E}[d_t] + O(\alpha^2 K \sigma^2 ) + \frac{\alpha B_t}{\beta} + \mathbb{E}[M_t]\\
&\leq (1- \alpha \beta + C_1 \alpha^2 \tau') \mathbb{E}[d_t] + O(\alpha^2 \tau' K \sigma^2 ) + \frac{\alpha B_t}{\beta}\\
& \leq \left(1 - \frac{\alpha \beta}{2} \right) \mathbb{E}[d_t] + O(\alpha^2 \tau' K \sigma^2 ) + \frac{\alpha B_t}{\beta},
\end{aligned}
\label{eqn:finalbnd1}
\end{equation}
where in the second inequality, $C_1$ is some universal constant, and the last inequality results from picking $\alpha$ to satisfy 
$$ \alpha \leq \frac{\beta}{2 C_1 \tau'}.$$
Recalling that $$ B_t = c\left(\epsilon + \frac{\log^2(KT) \tau_{mix}}{t} \right) K \sigma^2_1,$$
and unrolling the inequality in Eq.~\eqref{eqn:finalbnd1} starting from $t= \bar{T}+\tau'$ yields: 
\begin{equation}
\begin{aligned}
\mathbb{E}[d_T] &\leq \left(1- \frac{\alpha \beta}{2} \right)^{T- \bar{T}-\tau'} \mathbb{E}[d_{\bar{T}+\tau'}] + O(\alpha \tau_{mix} K \sigma^2) \frac{\log^2(KT)}{\beta} \sum_{k= \bar{T}+\tau'}^{T-1} \left(1- \frac{\alpha \beta}{2} \right)^{T- 1-k} \frac{1}{k} \\
& +O(\alpha^2 \tau' K \sigma^2) \sum_{k=0}^{\infty} \left(1- \frac{\alpha \beta}{2} \right)^k + O\left( \frac{\alpha K \epsilon \sigma^2_1}{\beta} \right) \sum_{k=0}^{\infty} \left(1- \frac{\alpha \beta}{2} \right)^k\\
& \leq \left(1- \frac{\alpha \beta}{2} \right)^{T- \bar{T}-\tau'} \mathbb{E}[d_{\bar{T}+\tau'}] +  O\left(\frac{\alpha \tau_{mix} K \log^2(KT) \sigma^2}{\beta} \right) \sum_{k= 1}^{T-1} \frac{1}{k} +  O\left( \frac{\alpha \tau' K \sigma^2}{\beta} \right) + O\left( \frac{K \epsilon \sigma^2_1}{\beta^2} \right)\\
& \leq \left(1- \frac{\alpha \beta}{2} \right)^{T- \bar{T}-\tau'} \mathbb{E}[d_{\bar{T}+\tau'}] + O\left(\frac{\alpha \bar{\tau}_{mix} K \log^3(KT) \sigma^2}{\beta} \right) + O\left( \frac{K \epsilon \sigma^2_1}{\beta^2} \right),
\end{aligned}
\end{equation}
where recall that $\bar{\tau}_{mix}=\max\{\tau', \tau_{mix}\}.$ To arrive at the last step, we used $\sum_{k=1}^{T-1} (1/k) = O(\log(T)).$ Now suppose $T$ is chosen sufficiently large such that $\bar{T} \leq T/4$ and $\tau' \leq T/4$. Then, we have:
\begin{equation}
\begin{aligned}
    &\mathbb{E}[{d_T}] \leq \underbrace{\left(1- \frac{\alpha \beta}{2} \right)^{T/2} \mathbb{E}[d_{\bar{T}+\tau'}]}_{(*)} + \underbrace{O\left(\frac{\alpha \bar{\tau}_{mix} K \log^3(KT) \sigma^2}{\beta} \right)}_{(**)} + O\left( \frac{K \epsilon \sigma^2_1}{\beta^2} \right).
\end{aligned}
\label{eqn:finalbnd2}
\end{equation}
Now let us substitute $\alpha = \frac{4}{\beta} \frac{\log(T)}{T}$ in the above bound. With this choice of $\alpha$, we have
$$(**) \leq \tilde{O}\left( \frac{K \bar{\tau}_{mix} \sigma^2}{\beta^2 T} \right). $$
We claim that $(*) \leq O(\sigma^2 K/T).$ To see why, start by noting that based on our choice of $\alpha$:
$$ \left(1- \frac{\alpha \beta}{2} \right)^{T/2} \leq \exp\left( - \frac{\alpha \beta T}{4} \right) = \exp (-\log(T)) = \frac{1}{T}.$$
To establish the claim, it remains to argue that $\mathbb{E}[d_{\bar{T}+ \tau'}] \leq O(K \sigma^2).$ To that end, using the same reasoning as we did to arrive at Eq.~\eqref{eqn:drift_int}, we have
\begin{equation}\begin{aligned}\Vert \theta_{\bar{T}+\tau'} \Vert  \leq 2 \Vert \theta_{\bar{T}} \Vert + O(\sqrt{K} \sigma) &= 2 \Vert \theta_{0} \Vert + O(\sqrt{K} \sigma)\\
&= O(\sqrt{K} \sigma).\end{aligned}\end{equation}
Here, we used that for $t \leq \bar{T}, \theta_t = \theta_0$, and $\Vert \theta_0 \Vert \leq \sigma.$ Thus, $d_{\bar{T}+\tau'} = \Vert \theta_{\bar{T}+\tau'}  - \theta^* \Vert^2 \leq 2 \Vert \theta_{\bar{T}+\tau'} \Vert^2 + 2 \Vert \theta^* \Vert^2 \leq O(K \sigma^2).$ Combining all the pieces together, we have
\begin{equation}\begin{aligned}
&\mathbb{E}[d_T] \leq \tilde{O}\left( \frac{\bar{\tau}_{mix } \sigma^2 G}{T} \right) + O\left(\epsilon \sigma^2_1 G\right), \\
&\textrm{where} \hspace{1mm} G = \frac{K}{\omega^2 (1-\gamma)^2}.\end{aligned}\end{equation}

To complete the proof, it remains to specify the parameters $\bar{T}$ and $T$. As for the burn-in time $\bar{T}$, we note from the analysis of Lemma~\ref{lemm:rob_est} that the following choice of $\bar{T}$ suffices:
$$ \bar{T} = \ceil{ c_1 \tau_{mix} \log^2(KT)}.$$

Next, all the requirements on the step-size $\alpha$ needed to arrive at our final bound can be subsumed into the following requirement:
$$ \alpha \leq \frac{\omega (1-\gamma)}{ C' \tau'},$$
where $C' \geq 8$ is some suitably large universal constant. Now, since we have fixed $\alpha $ to be $\frac{4}{\beta} \frac{\log(T)}{T}$, the above criterion can be met, provided the number of iterations $T$ satisfies:
$$ T \geq \frac{ 4 C' \tau' \log(T)}{\omega^2 (1-\gamma)^2}.$$
The above requirement on $T$, combined with the fact that we need $T \geq \bar{T} + \tau'$, justifies the choice of $T$ in the statement of Theorem~\ref{thm:main_result}.
\end{proof}
\section{Proof of Lower Bound in Theorem~\ref{thm:lowerbnd}}
\label{app:lowerbnd}
In this section, we will establish the lower bound in Theorem~\ref{thm:lowerbnd}. Let us start by explaining the high-level idea behind our proof, and then we will supply all the technical details. 

The main idea is to construct two different Markov Reward Processes (MRPs) induced by the same policy, such that (i) the value functions induced by the policy differ in magnitude by $\Omega(\sqrt{\epsilon})$ in the two MRPs; and (ii) the distribution of rewards under the Huber-contaminated observation model is identical across the two MRPs. It is then not too hard to argue that any estimator for a value function must suffer an error of  $\Omega(\sqrt{\epsilon})$ on at least one of the two MRPs. We now proceed to construct our hard instance.

\textbf{Step 1: Construction of the MRPs.} Consider a MDP with just one state $s$ and one action $a$. Trivially, a policy $\mu$ thus maps $s$ to $a$, and there is no randomness in terms of state transitions. We will now construct two MRPs, MRP 1 and MRP 2, induced by the policy $\mu$, that differ in terms of their noisy reward models. For MRP 1, the reward random variable $r_1(s)$ has support comprising two values: 
\begin{equation}\label{eqn:D1}
\begin{aligned}
r_1(s)=\begin{cases}
      \frac{\rho}{\sqrt{\epsilon}} & \text{with probability $\frac{\epsilon}{4 (1-\epsilon)}$}\\
      0 & \text{with probability $1 - \frac{\epsilon}{4 (1-\epsilon)}$},\\
    \end{cases} 
\end{aligned}
\end{equation}
where $\rho > 0$ is some positive constant. We call this reward distribution $\mc{D}_1$. For MRP 2, the reward random variable $r_2(s)$ has distribution $\mc{D}_2$ defined similarly as follows:
\begin{equation}\label{eqn:D2}
\begin{aligned}
r_2(s)=\begin{cases}
      - \frac{\rho}{\sqrt{\epsilon}} & \text{with probability $\frac{\epsilon}{4 (1-\epsilon)}$}\\
      0 & \text{with probability $1 - \frac{\epsilon}{4 (1-\epsilon)}$}.\\
    \end{cases} 
\end{aligned}
\end{equation}
Let the mean of the rewards under $\mc{D}_1$ and $\mc{D}_2$ be denoted by $R_1$ and $R_2$, respectively. It is then easy to see that
$$ R_1 = \frac{\rho \sqrt{\epsilon}}{ 4 (1-\epsilon)}, \hspace{1mm} \textrm{and} \hspace{1mm} R_2 = - \frac{\rho \sqrt{\epsilon}}{ 4 (1-\epsilon)}.$$
Furthermore, the variance of both $r_1(s)$ and $r_2(s)$ is given by
$$ Var(r_1(s)) = Var(r_2(s)) \leq \frac{\rho^2}{\epsilon} \times \frac{\epsilon}{4 (1-\epsilon)} < 0.5 \rho^2,$$
where we used the fact that the corruption fraction $\epsilon$ satisfies $\epsilon < 0.5$. 
 Thus, each reward model has a finite variance bounded above by $ \rho^2$. It is easily seen that the value functions in the two MRPs, say $V_1$ and $V_2$, satisfy:\footnote{We drop the dependence of $V_i$ on $s$ since there is only one state.}
\begin{equation}
V_i = \frac{R_i}{(1-\gamma)}, i \in \{1, 2\}. 
\label{eqn:vfuncs} 
\end{equation}

\textbf{Step 2: Construction of the Attack Distributions.} Consider an error distribution $\mc{Q}_1$ associated with MRP 1 such that a random variable $Z_1 \sim \mc{Q}_1$ is given by 
\begin{equation}\label{eqn:Q1}
\begin{aligned}
Z_1=\begin{cases}
      - \frac{\rho}{\sqrt{\epsilon}} & \text{with probability $\frac{1}{2}$} \\
      0 & \text{with probability $\frac{1}{4}$},\\
      \frac{\rho}{\sqrt{\epsilon}} & \text{with probability $\frac{1}{4}$}.
    \end{cases} 
\end{aligned}
\end{equation}

Now consider a random variable $X$ drawn from the Huber-contaminated mixture model $(1-\epsilon) \mc{D}_1 + \epsilon \mc{Q}_1$. Given the distributions of $\mc{D}_1$ and $\mc{Q}_1$, one can verify (with straightforward calculations) that the distribution of $X$ is as follows:
\begin{equation}
\label{eqn:distX}
\begin{aligned}
X=\begin{cases}
      - \frac{\rho}{\sqrt{\epsilon}} & \text{with probability $\frac{\epsilon}{2}$} \\
      0 & \text{with probability $1 - \epsilon$},\\
      \frac{\rho}{\sqrt{\epsilon}} & \text{with probability $\frac{\epsilon}{2}$}. 
    \end{cases} 
\end{aligned}
\end{equation}

For MRP 2, we construct an error distribution $\mc{Q}_2$ such that a random variable $Z_2$ drawn from $\mc{Q}_2$ is as follows:
\begin{equation}
\begin{aligned}
Z_2=\begin{cases}
      - \frac{\rho}{\sqrt{\epsilon}} & \text{with probability $\frac{1}{4}$} \\
      0 & \text{with probability $\frac{1}{4}$},\\
      \frac{\rho}{\sqrt{\epsilon}} & \text{with probability $\frac{1}{2}$}.
    \end{cases} 
\end{aligned}
\end{equation}
Now consider a random variable $Y$ drawn as per $(1-\epsilon) \mc{D}_2 + \epsilon \mc{Q}_2$. The specific nature of our construction ensures that 
$$(1-\epsilon) \mc{D}_1 + \epsilon \mc{Q}_1 = (1-\epsilon) \mc{D}_2 + \epsilon \mc{Q}_2.$$

In summary, the contaminated reward random variable $X$ in MRP 1 has the same distribution as the contaminated reward random variable $Y$ in MRP 2. As such, these two reward models are indistinguishable to a learner. However, we also have: 
\begin{equation}
|V_1 - V_2| = \frac{\rho \sqrt{\epsilon}}{ 2 (1-\epsilon) (1-\gamma)} \geq \frac{\rho \sqrt{\epsilon}}{2 (1-\gamma)}.
\label{eqn:Vdiff}
\end{equation}
 In light of the above facts, we now proceed to argue that any value-function estimator must suffer $\Omega\left(\frac{\rho \sqrt{\epsilon}}{(1-\gamma)}\right)$ error in at least one of the MRPs. 

\textbf{Step 3. Lower-bounding Error of any Estimator.} For $i= 1, \ldots, T$, let $(X_i, Y_i)$ be independent pairs of random observations satisfying:
$$ \mathbb{P}(X_i = Y_i = - \rho/{\sqrt{\epsilon}}) = \frac{\epsilon}{2}, $$ 
$$\, \mathbb{P}(X_i = Y_i = 0) = 1- \epsilon \, , \, \mathbb{P}(X_i = Y_i =  \rho/{\sqrt{\epsilon}}) = \frac{\epsilon}{2}. $$
Let us note that $X_i$ is distributed as per $(1-\epsilon) \mc{D}_1 + \epsilon \mc{Q}_1$, and $Y_i$ as per $(1-\epsilon) \mc{D}_2 + \epsilon \mc{Q}_2$. Clearly, the following is true: $\mathbb{P}\left(\{X_i\}_{i \in [T]} = \{Y_i\}_{i \in [T]}\right) =1.$ Now suppose $\hat{R}_T$ is any estimator for estimating the means of the rewards in the two MRPs. As we shall see, establishing a fundamental limit on the performance of $\hat{R}_T$ is sufficient to establish a limit on the performance of any value-function estimator. In what follows, for conciseness of notation, let
$$ B \triangleq \frac{\rho \sqrt{\epsilon}}{ 4 (1-\epsilon)}. $$

We then have
\begin{equation}\medmath{
\begin{aligned}
&\max\left\{ \mathbb{P}\left( \vert \hat{R}_T(\{X_i\}_{i \in [T]}) - R_1 \vert > \frac{B}{2} \right), \mathbb{P}\left( \vert \hat{R}_T(\{Y_i\}_{i \in [T]}) - R_2 \vert > \frac{B}{2} \right) \right\} \\
& \geq \frac{1}{2} \mathbb{P} \left( \left\{\vert \hat{R}_T(\{X_i\}_{i \in [T]}) - R_1 \vert > \frac{B}{2} \right\} \bigcup \left\{ \vert \hat{R}_T(\{Y_i\}_{i \in [T]}) - R_2 \vert > \frac{B}{2} \right \} \right) \\
& \geq \frac{1}{2} \mathbb{P}\left(
\hat{R}_T(\{X_i\}_{i \in [T]}) = \hat{R}_T(\{Y_i\}_{i \in [T]})
\right)\\
& \geq \frac{1}{2} \mathbb{P}\left(
\{X_i\}_{i \in [T]} = \{Y_i\}_{i \in [T]}
\right)\\
& = \frac{1}{2},
\end{aligned}}
\end{equation}
where for the second inequality, we used $R_1 =B$ and $R_2 = -B$. Using $1/(1-\epsilon) > 1$, we then conclude that:
\begin{equation}\label{eqn:prob_bnd}
    \begin{aligned}
        &\max\left\{ \mathbb{P}\left( \left| \hat{R}_T(\{X_i\}_{i \in [T]}) - R_1 \right| > \frac{\rho \sqrt{\epsilon}}{8} \right), \right.\left. \mathbb{P}\left( \left| \hat{R}_T(\{Y_i\}_{i \in [T]}) - R_2 \right| > \frac{\rho \sqrt{\epsilon}}{8} \right) \right\} \geq \frac{1}{2}.
    \end{aligned}
\end{equation}
Let $\hat{V}_T$ be any estimator for the value functions in the two MRPs. In light of Eq.~\eqref{eqn:prob_bnd}, we claim that 
\begin{equation}\label{eqn:prob_bnd2}
    \begin{aligned}
        &\max\left\{ \mathbb{P}\left( \vert \hat{V}_T(\{X_i\}_{i \in [T]}) - V_1 \vert > \frac{\rho \sqrt{\epsilon}}{8 (1-\gamma) } \right), \right.\left. \mathbb{P}\left( \vert \hat{V}_T(\{Y_i\}_{i \in [T]}) - V_2 \vert > \frac{\rho \sqrt{\epsilon}}{8 (1-\gamma) } \right) \right\} \geq \frac{1}{2}.
    \end{aligned}
\end{equation}
The claim essentially follows from the simple observation that if a value-function estimator $\hat{V}_T$ can accurately estimate both $V_1$ and $V_2$, then one can use such an estimator to construct accurate estimates of both $R_1$ and $R_2$, thereby violating Eq.~\eqref{eqn:prob_bnd}. Formally, to see that Eq.~\eqref{eqn:prob_bnd} implies Eq.~\eqref{eqn:prob_bnd2}, suppose there exists an estimator $\hat{V}_T$ such that
\begin{equation}
    \begin{aligned}
        \max\Bigg\{ &\mathbb{P}\left( \left| \hat{V}_T(\{X_i\}_{i \in [T]}) - V_1 \right| > \frac{\rho \sqrt{\epsilon}}{8 (1-\gamma) } \right),
        &\mathbb{P}\left( \left| \hat{V}_T(\{Y_i\}_{i \in [T]}) - V_2 \right| > \frac{\rho \sqrt{\epsilon}}{8 (1-\gamma) } \right)
        \Bigg\} < \frac{1}{2}.
    \end{aligned}
\end{equation}
Using $\hat{V}_T$, construct a reward estimator $\hat{R}_T = (1-\gamma) \hat{V}_T$. From Eq.~\eqref{eqn:vfuncs}, we then immediately have:
\begin{equation}
    \begin{aligned}
        \max\Bigg\{ &\mathbb{P}\left( \left| \hat{R}_T(\{X_i\}_{i \in [T]}) - R_1 \right| > \frac{\rho \sqrt{\epsilon}}{8} \right), \mathbb{P}\left( \left| \hat{R}_T(\{Y_i\}_{i \in [T]}) - R_2 \right| > \frac{\rho \sqrt{\epsilon}}{8} \right)
        \Bigg\} < \frac{1}{2}.
    \end{aligned}
\end{equation}

This completes the claim and the proof. 
\begin{figure}[t]
\begin{tabular}{cc}
\hspace{15 mm} \includegraphics[scale=0.45]{vulnerabilityexp1.pdf}& \hspace{-3mm}  \includegraphics[scale=0.45]{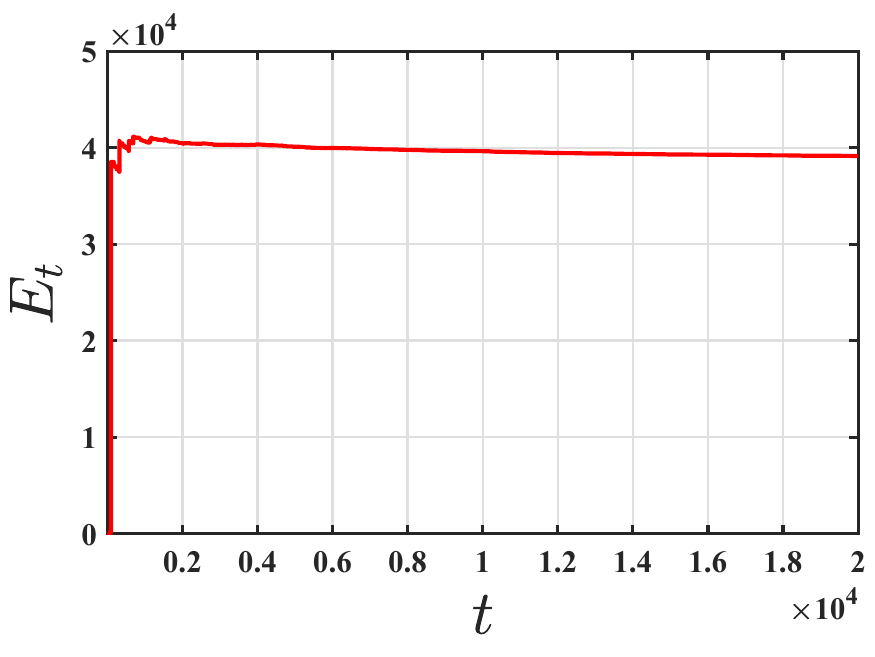}
\end{tabular}
\caption{Plots of the mean-square error $E_t=\Vert \theta_T-\theta^* \Vert^2_2$ for \texttt{TD}(0) under the Huber-contaminated reward model with corruption probability $\epsilon=0.001$, and a simple biasing attack where the attack signal is $100/\epsilon.$
(\textbf{Left}) Constant step-size $\alpha=0.1$. (\textbf{Right}) Diminishing step-size $\alpha_t =1/t$.} 
\label{fig:sim1}
\end{figure}

\begin{figure}[t]
\begin{tabular}{cc}
\hspace{15 mm} \includegraphics[scale=0.45]{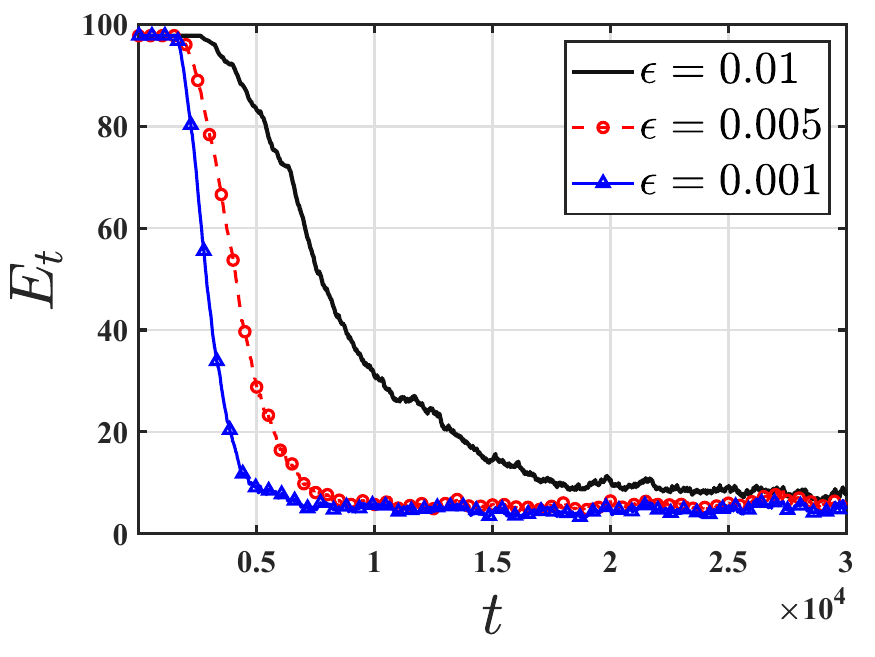}& \hspace{-3mm}  \includegraphics[scale=0.45]{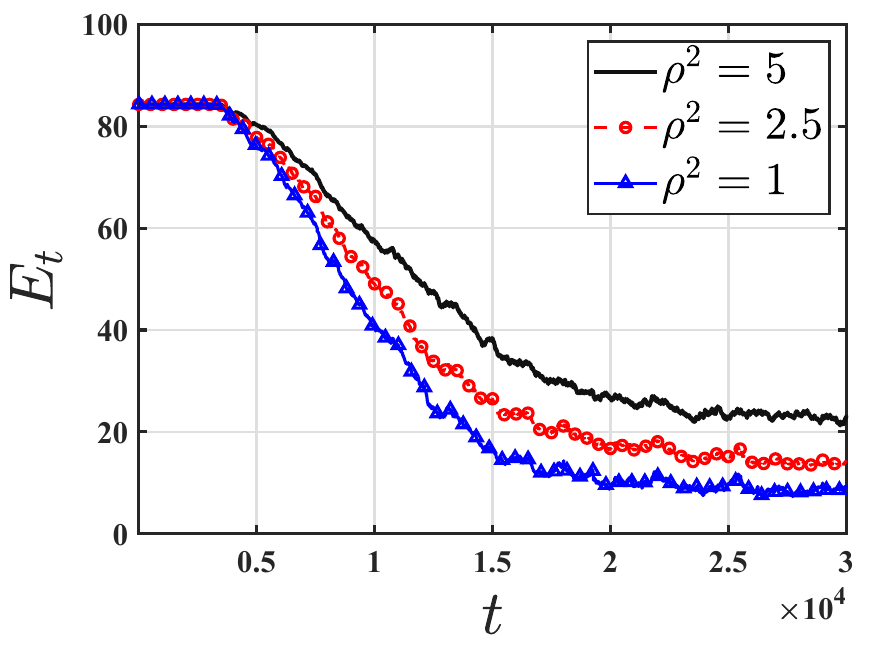}
\end{tabular}
\caption{Plots of the mean-square error $E_t=\Vert \theta_T-\theta^* \Vert^2_2$ for \texttt{Robust-TD} under the Huber-contaminated reward model with a biasing attack, where the attack signal is $100/\epsilon.$
(\textbf{Left}) Effect of varying the corruption probability $\epsilon$. (\textbf{Right}) Effect of varying the noise variance $\rho^2$.}  
\label{fig:sim2}
\end{figure}
\section{Simulation Study}
\label{app:Sims}
In this section, we report some synthetic experiments to support the theory developed in this paper. All simulations are performed on an HP Spectre x$360$ personal laptop with $11$th Gen Intel(R) $4$-Core Processor. 

\textbf{Basic Setup.} We construct an MDP with $100$ states, and use a feature matrix $\Phi$ with $K=10$ independent basis vectors. Using this MDP, we generate the state transition matrix $P_{\mu}$ and reward vector $R_{\mu}$ associated with a fixed policy $\mu$. For all our simulations, the discount factor $\gamma$ is set to $0.5$, and the rewards are generated uniformly at random from the interval $(0,5)$. Unless specified, the step size $\alpha$ is chosen to be $0.1$. We perform $10$ independent trials per simulation and average the errors from each trial to report the mean-square error $E_t=\Vert \theta_T-\theta^* \Vert^2_2$. With this basic setup, we now report various experiments below.

\begin{enumerate}
\item \textbf{Vulnerability of \texttt{TD}(0)}. The purpose of the first simulation is to reveal the vulnerability of the basic \texttt{TD}(0) algorithm to adversarial reward contamination. To that end, we consider a scenario where the corruption fraction is $\epsilon=0.001$, and the rewards are noiseless. In each state $s$, with probability $\epsilon$, the adversary injects a biasing signal of magnitude $100/\epsilon.$ The outcome of this experiment is reported in  Fig.~\ref{fig:sim1}. When the step size is held constant at $\alpha=0.1$, convergence is to a ball centered around a perturbed parameter $\tilde{\theta}^*$, where $\tilde{\theta}^*$ is as in Theorem~\ref{thm:vulnerability}. The size of this ball scales with the effective reward magnitude that depends on the bias $100/\epsilon$. Hence, in the left display of Fig.~\ref{fig:sim1}, we see large oscillations that depend on the bias magnitude. With a diminishing step-size of the form $\alpha_t =1/t$, Theorem~\ref{thm:vulnerability} suggests exact convergence to the perturbed point $\tilde{\theta}^*$. This is reflected in the right display of Fig.~\ref{fig:sim1}, where the mean-square error (MSE) converges to a steady-state value that is bounded far away from $0$. 

\item \textbf{Performance of \texttt{Robust-TD}}. In our next simulation, we assess the performance of \texttt{Robust-TD}. For our first experiment, the noise model comprises a zero-mean Gaussian distribution with a variance of 1. We consider the same biasing attack as before, where the attacker injects a constant bias of $100/\epsilon$. We vary the corruption probability $\epsilon$, and report our findings in the left display of Fig.~\ref{fig:sim2}. For each value of $\epsilon$, the MSE converges to a small ball around $0$. In the next experiment, we fix the corruption probability to $0.01$ and vary the noise variance level. We consider three different values of variance: $5$, $2.5$, and $1$. As expected, with a constant step size, the MSE settles down to a ball around the origin, where the size of the ball depends on the noise variance. Notably, complying with Theorem~\ref{thm:main_result}, the MSE of \texttt{Robust-TD} is unaffected by the magnitude of the adversarial bias input. 
\end{enumerate}
\end{document}